%% file: icml2025_main_revision.tex
\documentclass{article}

\usepackage{microtype}
\usepackage{graphicx}
\usepackage{tikz}
\usetikzlibrary{calc}

\usepackage{enumitem}
\usepackage{booktabs} 
\usepackage{multirow}
\usepackage{subcaption}
\usepackage{graphicx}
\usepackage{caption}
\usepackage{siunitx}
\usepackage{hyperref}
\usepackage[list=true, labelfont=bf, 
labelformat=brace, position=top]{subcaption}

\usepackage[accepted]{icml2025}

\usepackage{amsmath}
\usepackage{amssymb}
\usepackage{mathtools}
\usepackage{amsthm}

\definecolor{Maroon}{rgb}{0.5, 0.0, 0.0}
\definecolor{Darkorange}{rgb}{1.0, 0.55, 0.0}
\definecolor{Darkmidnightblue}{rgb}{0.0, 0.2, 0.4}
\definecolor{Darkspringgreen}{rgb}{0.09, 0.45, 0.27}
\definecolor{Airforceblue}{rgb}{0.36, 0.54, 0.66}
\definecolor{Sensorgreen}{cmyk}{0.1,0,1,0}
\definecolor{Fluxbc}{rgb}{1,0,0}

\DeclarePairedDelimiter{\prt}{(}{)}
\DeclarePairedDelimiter{\brk}{[}{]}

\def\abs#1{\left|#1\right|}
\def\norm#1{\left\|#1\right\|}
\newcommand{\flow}{\mathbf{z}}
\newcommand{\R}{\mathbb{R}}

\renewcommand{\L}{\mathcal{L}}
\newcommand{\eps}{\varepsilon}
\newcommand{\V}{\mathcal V}
\newcommand{\E}{\mathcal E}
\newcommand{\K}{\mathcal K}
\newcommand{\A}{\mathcal A}
\newcommand{\D}{\mathcal D}

\newcommand{\ubase}{u^{\textup{sensor}}}
\newcommand{\usensor}{u^{\textup{sensor}}}
\newcommand{\uinflow}{u^{\textup{inflow}}}
\newcommand{\uoutflow}{u^{\textup{outflow}}}
\newcommand{\uorigin}{u^{\textup{origin}}}
\newcommand{\utarget}{u^{\textup{target}}}
\newcommand{\uinit}{u^{\textup{init}}}

\newcommand{\npdebatch}{n_{\text{pde}}}
\newcommand{\ninitbatch}{n_{\text{init}}}
\newcommand{\nbcbatch}{n_{\text{bc}}}

\newcommand{\torigin}{t^\text{origin}}
\newcommand{\ttarget}{t^\text{target}}
\newcommand{\xinit}{x^\text{init}}

\newcommand{\pinn}{\textsc{PINN}}
\newcommand{\jax}{\textsc{JAX}}
\newcommand{\deeponet}{\textsc{DeepONet}}
\newcommand{\deeponets}{\textsc{DeepONets}}
\newcommand{\pinns}{\textsc{PINNs}}

\newcommand{\KN}{Kirchhoff-Neumann}

\usepackage[capitalize,noabbrev]{cleveref}

\theoremstyle{plain}
\newtheorem{theorem}{Theorem}[section]
\newtheorem{proposition}[theorem]{Proposition}
\newtheorem{lemma}[theorem]{Lemma}

\theoremstyle{definition}

\theoremstyle{remark}
\newtheorem{remark}[theorem]{Remark}

\usepackage[textsize=tiny]{todonotes}

\icmltitlerunning{Physics-Informed DeepONets for drift-diffusion on metric graphs}

\begin{document}
	
	\twocolumn[
	\icmltitle{Physics-Informed DeepONets for drift-diffusion on metric graphs: simulation and parameter identification}
	
	
	
	\icmlsetsymbol{equal}{*}
	
	\begin{icmlauthorlist}
\icmlauthor{Jan, Blechschmidt}{yyy}
\icmlauthor{Tom-Christian, Riemer}{yyy}
\icmlauthor{Max, Winkler}{yyy}
\icmlauthor{Martin, Stoll}{yyy}
\icmlauthor{Jan-F., Pietschmann}{zzz}
	\end{icmlauthorlist}
	
	\icmlaffiliation{yyy}{Department of Mathematics, TU Chemnitz, Chemnitz, Germany}
	\icmlaffiliation{zzz}{Department of Mathematics, University of Augsburg, Augsburg, Germany and Centre for Advanced Analytics and Predictive Sciences (CAAPS), University of Augsburg,
Universit\"{a}tsstr. 12a, 86159 Augsburg, Germany. }
	
	\icmlcorrespondingauthor{Firstname1 Lastname1}{first1.last1@xxx.edu}
	\icmlcorrespondingauthor{Firstname2 Lastname2}{first2.last2@www.uk}
	
	\icmlkeywords{Machine Learning, ICML}
	
	\vskip 0.3in
	]
	
	
	
	
	\begin{abstract}
		We develop a novel physics informed deep learning approach for solving nonlinear drift-diffusion equations on metric graphs.
		These models represent an important model class with a large number of applications in areas ranging from transport in biological cells to the motion of human crowds.
		While traditional numerical schemes require a large amount of tailoring, especially in the case of model design or parameter identification problems, physics informed deep operator networks (\deeponets ) have emerged as a versatile tool for the solution of partial differential equations with the particular advantage that they easily incorporate parameter identification questions.
		We here present an approach where we first learn three \deeponet\ models for representative inflow, inner and outflow edges, resp., and then subsequently couple these models for the solution of the drift-diffusion metric graph problem by relying on an edge-based domain decomposition approach.
		We illustrate that our framework is applicable for the accurate evaluation of graph-coupled physics models and is well suited for solving optimization or inverse problems on these coupled networks.

	\end{abstract}
	
\section{Introduction}
Dynamic processes on graphs \cite{newman2018networks,barabasi2013network} are crucial for understanding complex phenomena in many application areas. 
We focus on the case of a metric graph where each edge is associated with an interval of (possibly) different length. Therefore, the metric graph can be  equipped with a differential operator, acting separately on each edge and with appropriate coupling conditions or boundary conditions at the nodes, called the \textit{Hamiltonian}, leading to what is known as \textit{quantum graphs} \cite{lagnese2012modeling,berkolaiko2013introduction}.
Numerical methods for quantum graphs have gained recent interest \cite{arioli2018finite, gyrya2019explicit, stoll2021optimal} both for simulation of PDE models as well as for solving design or inverse problems. As the structure of such graphs is typically rather complex, efficient schemes such as domain decomposition methods \cite{leugering2017domain} are often needed for computational efficiency.

\begin{figure}
\begin{tikzpicture}[scale=1.0]
    \tikzset{known/.style={thick}};
    \tikzset{unknown/.style={thick, dashed, red}};
    \tikzset{inflow/.style={shape=circle,draw=Darkspringgreen,fill=Darkspringgreen,fill opacity=.7,text opacity=1.}};
    \tikzset{inner/.style={shape=circle,draw=Airforceblue,fill=Airforceblue,fill opacity=.6,text opacity=1.}};
    \tikzset{sensor/.style={shape=rectangle,draw=Darkorange,fill=Sensorgreen,fill opacity=.6,text opacity=1.}};
    \tikzset{outflow/.style={shape=circle,draw=Maroon,fill=Maroon,fill opacity=.5,text opacity=1.}};

    \begin{scope}[y={(-.8cm,0.5cm)},x={(1cm,0.5cm)}, z={(0cm,1cm)}]
    \node[inflow] (v1) at (-2.4,1.4,0) {$v_1$};
    \node[outflow] (v5) at (2.6,1.1,0) {$v_5$};
    \node[inner] (v3) at (-1,0,0) {$v_3$};
    \node[inner] (v4) at (1,0,0) {$v_4$};
    \node[inflow] (v2) at (-2.4,-.95,0) {$v_2$};
    \node[outflow] (v6) at (2.4,-1.4,0) {$v_6$};

    \coordinate (e3m) at (.20,-.2); 
    \coordinate (eee) at (-1.7,-3.7,0);


    \path [very thick, ->, color=Darkspringgreen](v1) edge node[below, transform canvas={xshift=.2cm}] {$e_1$} (v3);
    \path [very thick,->, color=Darkspringgreen](v2) edge node[below, transform canvas={xshift=.2cm}] {$e_2$} (v3);
    \path [very thick,->, color=Airforceblue](v3) edge node[below] {$e_3$} (v4);
    \path [very thick,->,color=Maroon](v4) edge node[below, transform canvas={xshift=.2cm}] {$e_4$} (v5);
    \path [very thick,->, color=Maroon](v4) edge node[below, transform canvas={xshift=-.2cm}] {$e_5$} (v6);

    

    \node (rect) [rectangle, draw, minimum width=40mm, minimum height=40mm, anchor= south west, rounded corners=5pt,dashed] at (-6,-8.0) {\includegraphics[width=.2\textwidth]{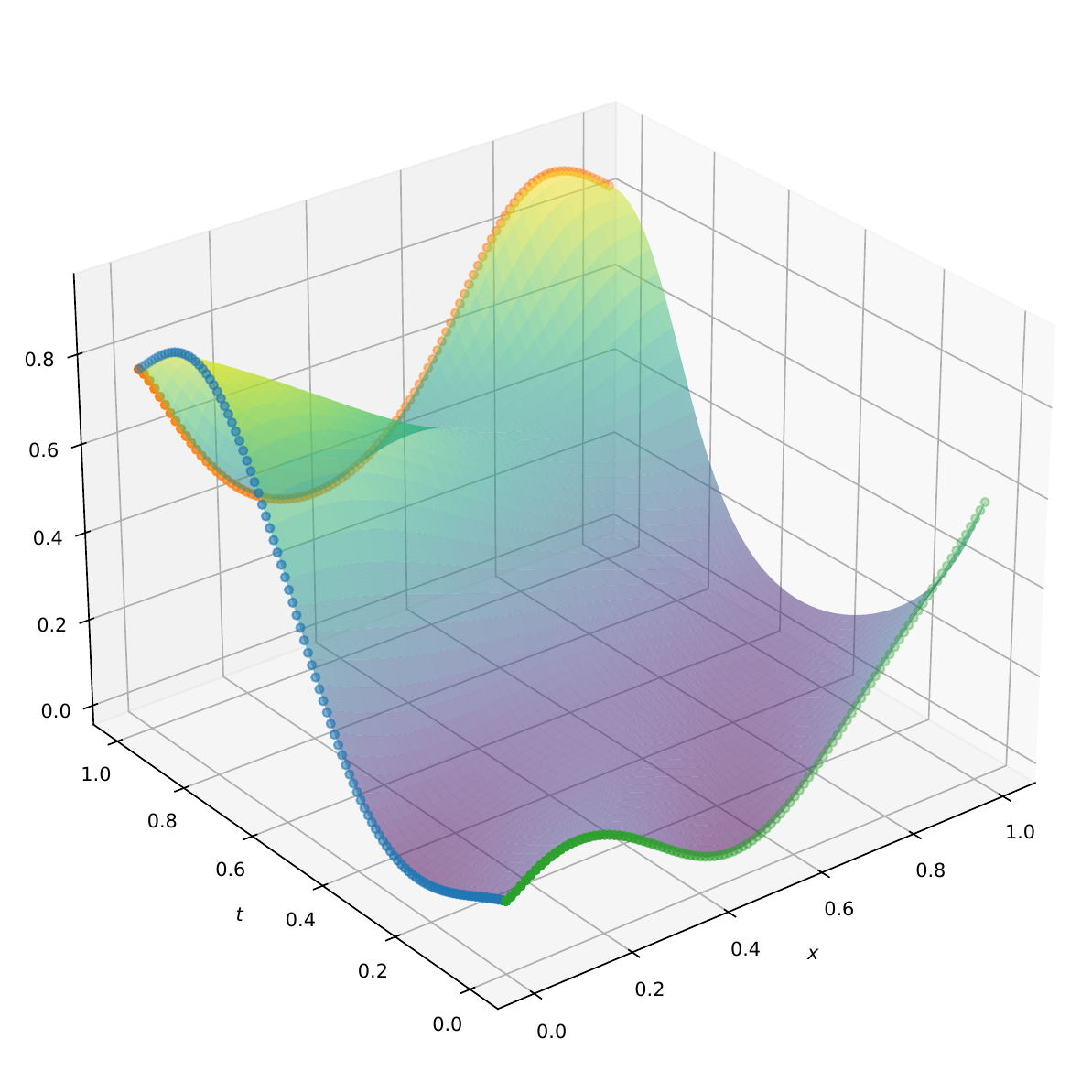}};
    \node[anchor=north] at (rect.north) {inner edge \textsc{DeepONet}};

    \draw[->] (v3) -- (rect.west) node[midway, left] {$\uorigin$};
    \draw[->] (v4) -- (rect.north) node[midway, right] {$\uoutflow$};
    \draw[->, line width = .6mm] (eee) -- (e3m) node[midway, right] {$\rho_e^u$};

    \def\upval{2.5}
    \coordinate (v1up) at (-2.4,1.4,\upval);
    \coordinate (v5up) at (2.6,1.1,\upval);
    \coordinate (v3up) at (-1,0,\upval);
    \coordinate (v4up) at (1,0,\upval);
    \coordinate (v2up) at (-2.4,-0.95,\upval);
    \coordinate (v6up) at (2.4,-1.4,\upval);
    
    \coordinate (e1) at ($(v1)!0.5!(v3)$);
    \coordinate (e2) at ($(v2)!0.5!(v3)$);
    \coordinate (e3) at ($(v3)!0.5!(v4)$);
    \coordinate (e4) at ($(v4)!0.5!(v5)$);
    \coordinate (e5) at ($(v4)!0.5!(v6)$);

    \node[sensor] (e1up) at ($(v1up)!0.5!(v3up)$) {$S_1$};
    \node[sensor] (e2up) at ($(v2up)!0.5!(v3up)$) {$S_2$};
    \node[sensor] (e3up) at ($(v3up)!0.5!(v4up)$) {$S_3$};
    \node[sensor] (e4up) at ($(v4up)!0.5!(v5up)$) {$S_4$};
    \node[sensor] (e5up) at ($(v4up)!0.5!(v6up)$) {$S_5$};
    
    \draw[dashed] (e1) -- (e1up);
    \draw[dashed] (e2) -- (e2up);
    \draw[dashed] (e3) -- (e3up);
    \draw[dashed] (e4) -- (e4up);
    \draw[dashed] (e5) -- (e5up);
    
    \end{scope}
    \begin{scope}
        \def\legendy{-4.5}
        \def\legendstep{.7}
        \def\xpos{-1}
        \node (legend) [rectangle, draw, minimum width=35mm, minimum height=30mm, anchor= south west] at (-4,\legendy-2.6) {};
        
        \node[inflow]  at (-3.6,\legendy) {$v$};
        \node[text width=4cm]  at (\xpos,\legendy) {inflow vertex};
        \node[outflow]  at (-3.6,\legendy-\legendstep) {$v$};
        \node[text width=4cm]  at (\xpos,\legendy-\legendstep) {outflow vertex};
        \node[inner]  at (-3.6,\legendy-\legendstep-\legendstep) {$v$};
        \node[text width=4cm]  at (\xpos,\legendy-\legendstep-\legendstep) {inner vertex};
        \node[sensor]  at (-3.6,\legendy-\legendstep-\legendstep-\legendstep) {$S$};
        \node[text width=4cm]  at (\xpos,\legendy-\legendstep-\legendstep-\legendstep) {measurement \\ sensor};

    \end{scope}
\end{tikzpicture}\end{figure}
	
In this paper we propose a machine learning technique, namely, the physics-informed \deeponet\ approach \cite{lu2021learning,wang2021learning} for drift-diffusion on metric graphs. These methods have been introduced to improve on the performance of the, by now well established, physics-informed neural networks (\pinns) \cite{raissi2019physics}, which have found their way into many application areas \cite{zhu2019physics,jin2021nsfnets,sahli2020physics} including fluid dynamics \cite{raissi2018hidden,mao2020physics,lye2020deep,magiera2020constraint,wessels2020neural}, continuum mechanics and elastodynamics \cite{haghighat2020deep,nguyen2020deep,rao2020physics}, inverse problems \cite{meng2020composite,jagtap2020conservative}, fractional advection-diffusion equations \cite{pang2019fpinns}, stochastic advection-diffusion-reaction equations \cite{chen2019learning}, stochastic differential equations \cite{yang2020physics} and power systems \cite{misyris2020physics}. \textsc{XPINN}s (eXtended \pinns) are introduced in~\cite{JagtapKardiadakis2020} as a generalization of \pinns\ involving multiple neural networks allowing for parallelization in space and time via domain decomposition, see also \cite{Heinlein2021} for a review on machine learning approaches in domain decomposition. 
Due to its broad range of applications, the PINN approach helped to establish the field of scientific machine learning \cite{thiyagalingam2022scientific,rackauckas2020universal,cuomo2022scientific,blechschmidt2021three}. 
On the other hand, the \pinn\ approach often suffers from reduced accuracy when compared with classical numerical methods for differential equations.
Furthermore, it has to be retrained everytime when initial conditions, boundary conditions or parameters of the PDE change.
The \deeponet\ architecture was introduced based on the universal approximation theorem for operators and relies on two neural networks for learning a representation of the solution operator, namely a \textit{branch net} for the input variables, e.g., time $t$ and space $x$, and a second neural network called \textit{trunk net} encoding boundary and initial conditions conditions as well as other parameters of the underlying problem, e.g., a variable velocity, viscosity or heat conductivity. Similar to \textsc{XPINN}s \deeponet\ has been extended for a domain decomposition application \cite{yin2022interfacing} where the key component is the coupling condition between the different domains that are constructed during the domain partitioning. 
	
In this work we introduce the extension of the  \deeponet\ framework to graphs, particularly the application to the case of a drift-diffusion equation posed on a metric graphs. Drift-diffusion models are used in many application areas ranging from modeling electrical networks, \cite{hinze2011pod}, to simulation of traffic flow in cities, \cite{Coclite2005}, and thus serve as a relevant and sufficiently complex test case. 

On the metric graph, the domain is naturally composed of a possibly large number of domains, i.e., the different edges. However, depending on the coupling conditions at vertices, different type of models have to be learned which distinguishes our approach from classical domain decomposition methods. Once these models are trained, we are able to obtain solutions on virtually arbitrary graphs via a computationally cheap optimization of loss terms at the nodes which ensure the coupling conditions. This advantage becomes even more significant when considering parameter identification problems where traditional PDE optimization based approaches would require many solutions of forward and adjoint equations, \cite{de2015numerical}. In our setting, solving the inverse problem merely manifests itself in adding additional loss terms. Therefore, strikingly, the cost of solving the forward and the inverse problem are practically the same.
	
Our main contributions are as follows:
\begin{itemize}
\item We propose a methodology to solve PDEs on graphs using a novel Lego-like domain decomposition approach where graph edges are represented by \deeponet\ models. 
\item Graph-agnostic training of the edge surrogate \deeponet\ model based on inner, inflow and outflow edges. No additional training is required to couple these models for representing flows on arbitrarily complex graphs.
\item  The novel \deeponet\ architecture enables robust model evaluation but also allows the solution of optimization or inverse problems at almost no additional cost. This is exemplified on a parameter identification problem.
\end{itemize}
	
	
	\section{Drift-diffusion equations on metric graphs}%
	\label{sec:DD_on_graphs}
	Let us introduce our notion of a metric graph in more detail. A metric graph is a directed graph that consists of a set of vertices $\V$ and edges $\E$ connecting a pair of vertices denoted by $(v_e^{\operatorname{o}},v_e^{\operatorname{t}})$ where $v_e^{\operatorname{o}},v_e^{\operatorname{t}}\in \V$. Here $v_e^{\operatorname{o}}$ denotes the vertex at the origin while $v_e^{\operatorname{t}}$ denotes the terminal vertex. In contrast to combinatorial graphs a length $\ell_e$ is assigned to each edge $e\in \E.$  We identify each edge with a one-dimensional interval which allows for the definition of differential operators. The graph domain is then denoted by
	\begin{equation*}
		\Gamma := \bigotimes_{e\in\E}[0,\ell_e].
	\end{equation*}
	We also introduce a normal vector $n_e(v)$ defined as $n_e(v_e^{\operatorname{o}}) = -1$ and $n_e(v_e^{\operatorname{t}}) = 1$.  To prescribe the behavior at the boundary of the graph, we first subdivide the set of vertices $\V$ into the interior vertices $\mathcal{V}_\mathcal{K}$ and the exterior vertices $\mathcal{V}_\mathcal{D}$ as follows:
	\begin{itemize} 
		\item the set of interior vertices $v \in \mathcal{V}_\mathcal{K} \subset \mathcal{V}$ contains all vertices that are incident to at least one incoming edge and at least one outgoing edge (i.e. $\forall v \in \mathcal{V}_\mathcal{K} \; \exists \ e_1, e_2 \in \mathcal{E}$ such that $v^{\operatorname{t}}_{e_1} = v$ and $v^{\operatorname{o}}_{e_2} = v$),
		
		\item the set of exterior vertices $v \in \mathcal{V}_\mathcal{D} \coloneqq \mathcal{V} \setminus \mathcal{V}_\mathcal{K}$, contains vertices to which either only incoming or only outgoing edges are incident, i.e., either $v^{\operatorname{t}}_{e} = v$ or $v^{\operatorname{o}}_{e} = v$ holds $\forall e \in \mathcal{E}_v$ with $\mathcal{E}_v$ the edge set incident to vertex $v$.
		
	\end{itemize}
	
	The differential operator defined on each edge consists of the non-linear drift-diffusion equation given by
	\begin{equation}
		\label{eq:strong_pde}
		\mathcal{H}(\rho_e) := \partial_t\rho_e  - \partial_x (\eps\,\partial_x\rho_e - \nu_e  \, f(\rho_e)) = 0, \quad e \in \E,
	\end{equation}
	where $\rho_e : e \times (0,T) \to \R_+$ describes the concentration of some quantity on the edge $e \in \E$, $\nu_e > 0$ is an edge-dependent velocity and $\eps > 0$ a (typically small) diffusion constant.
 Furthermore, $f: \R_+ \to \R_+$ satisfies $f(0) = f(1) = 0$.
 This property ensures that solutions satisfy $0 \le \rho_e \le 1$ a.e. on each edge, see Theorem \ref{thm:existence}.
 By identifying each edge with an interval $[0,\ell_e]$, we define the flux as
	\begin{align} \label{eq:flux}
		J_e(x) := - \eps\,\partial_x \rho_e (x) + \nu_e \, f(\rho_e(x))\,.
	\end{align}
	A typical choice for $f$ used in the following is $f(\rho_e) = \rho_e(1-\rho_e)$.
	
 \begin{remark}
     Note that the choice $\nu_e>0$ results in the fact that the prefered direction of transport is encoded in the direction of the edge (on our directed graph). On the other hand, due to the additional diffusion contributions, the flux $J_e$, and thus the direction of mass transport on each edge, may change sign.
 \end{remark}
	To make \eqref{eq:strong_pde} a well-posed problem, we need to add initial-conditions as well as coupling conditions in the vertices.
 First we impose on each edge $e \in \mathcal{E}$ the following initial condition
	\begin{equation}
		\label{eq:initial_conditions}
		\rho_e \left( 0,x \right)  = u_{e}^\text{init} \left( x \right),  \quad \text{for almost all } x \in (0, \ell_e), e \in \E,
	\end{equation}
	with $u_{e}^\text{init}  \in L^2 \left( e \right) $.
	
	For vertices $v\in \V_\K\subset \V$, we apply \emph{homogeneous \KN\ conditions}, i.\,e., there holds
	\begin{equation}
		\label{eq:Kirchhoff_Neumann_condition}
		\sum_{e\in \E_v}J_e(v)\,n_e (v)=0,
	\end{equation}
	for almost every $t \in (0,T)$ and with $\E_v$ the edge set incident to the vertex $v$.
	Additionally, we ask the solution to be continuous over vertices, i.e.
	\begin{align}
		\label{eq:continuity_condition}
		\rho_e(v) = \rho_{e'}(v) \quad \text{ for all }v \in \V_\K,\; e,\,e' \in \E_v,
	\end{align}
    again for almost every $t \in (0,T)$. 
	In vertices $v\in \V_\D:=\V\setminus \V_\K$ the solution $\rho$ fulfills \emph{flux boundary conditions}
	\begin{equation}
		\label{eq:Dirichlet_conditions}
		\sum_{e\in \E_v}J_e(v)\, n_e (v)=-u^\text{inflow}_v(t) \, (1-\rho_v)
		+ u^\text{outflow}_v(t) \, \rho_v,
	\end{equation}
	where
	$u^\text{inflow}_v:(0,T) \to \R_+, \,u^\text{outflow}_v : (0,T) \to \R_+$, ${v \in \V_\D}$,
	are functions prescribing the rate of influx of mass into the graph as well as the velocity of mass leaving the graph at the boundary vertices. Note that this choice ensures that the bounds $0 \le \rho_e \le 1$ are preserved, while the total mass on the complete graph may change over time. In typical situations, boundary vertices are either of influx- or of outflux type, i.e. $u^\text{inflow}_v(t)u^\text{outflow}_v(t) = 0$ for all $v \in \V_\D$.

	The \KN\ conditions are the natural boundary conditions for the differential operator \eqref{eq:strong_pde}, as they ensure that mass enters or leaves the system only via the boundary nodes $\V_\D$ for which either $\uinflow_{v}$ or $\uoutflow_v$ is positive.

	Having introduced the complete continuous model, we state the following existence and uniqueness result, whose proof can be found in Appendix \ref{sec:proof_existence}, together with a detailed definition of the function spaces involved.
	\begin{theorem}\label{thm:existence}
		Let the initial data $\uinit \in L^2(\Gamma)$ satisfy $0 \le \uinit \le 1$
		a.e.\ on $\E$ and let nonnegative functions $\uinflow_{v}, \uoutflow_v \in L^\infty(0,T)$, $v\in \V_D$ and non-negative numbers $\nu_e$, $e\in\E$, be given. Then there exists a unique weak solution $\rho \in L^2(0,T; H^1(\Gamma)) \cap H^1(0,T; H^1(\Gamma)^*)$ s.t.
			\begin{multline}\label{eq:weak}
			\sum_{e \in \E} \int_e  \left(\partial_t \rho_e(t)\,\varphi_e +
			(\varepsilon\,\partial_x \rho_e(t) - \nu_e \, 
			f(\rho_e(t)))\,\partial_x \varphi_e \right)dx\\
			+\sum_{v \in \V_D} (-\uinflow_{v}(t) (1-\rho(t,v)) + \uoutflow_v(t) \rho(t,v))\varphi(v) = 0,
		\end{multline}
					for all test functions $\varphi \in H^1(\Gamma)$ and a.a.\ $t\in(0,T)$.
  Here $L^2$ denotes the space of square integrable functions. The space $H^1$ denotes the space of functions for which also the weak derivative is bounded in $L^2$ and with $(H^1)^*$ its dual space. The Bochner spaces contain time-dependent functions where for $u( t,x)$ to belong to, e.g. $L^2(0,T; H^1(\Gamma))$, the norm
		$$
		\int_0^T \|u(t,\cdot)\|_{H^1(\Gamma)}^2\;dx 
		$$
		has to be finite.
	\end{theorem}

\section{Learning surrogate models}%
\label{sec:surrogate_model}



We apply the operator learning approach to obtain models for the dynamics on edges, given initial and boundary data. 
Due to the boundary and flux conditions \eqref{eq:Kirchhoff_Neumann_condition}--\eqref{eq:Dirichlet_conditions} each graph can be partitioned into three types of edges: 
\begin{itemize}
    \item inflow edges originate in a vertex $v^{\operatorname{o}} \in \V_\D$ with $u_{v^{\operatorname{o}}}^\text{inflow}(t) \neq 0$ and terminate in a inner vertex,
    \item inner edges originate and terminate in inner vertices,
    \item outflow edges originate in an inner vertex and terminate in a vertex $v^{\operatorname{t}} \in \V_\D$ with $u_{v^{\operatorname{t}}}^\text{outflow}(t) \neq 0$.
\end{itemize}
In our framework we design one \deeponet\ model for each of these three different edges. Once trained, this will allow to construct a composite model using the \deeponet\ submodels for inflow, outflow and inner edges as building blocks of typical graphs. To be more precise, the PDE-describing sensor measurements $\ubase_e = (\uorigin_e, \utarget_e, \uinit_e, \nu_e)$ are edge-specific, since they have to accommodate for different types of flux conditions, either \KN\ conditions \eqref{eq:Kirchhoff_Neumann_condition} for inner edges or inflow and outflow conditions \eqref{eq:Dirichlet_conditions} for inflow and outflow edges, respectively.

The training data, i.e.  boundary and initial conditions, are assumed to be given, as a function of discrete time, in certain sensor locations, and are collected in a vector $\ubase \in \R^{n_\text{sensor}}$.
Therefore, a deep operator network maps $(\ubase, t,x)$ to the solution of the respective PDE on an individual edge with initial and boundary conditions encoded in $\ubase$. 
Our physics-informed \deeponet\ does so by incorporation of residual terms that involve the point-wise evaluation of the PDE as well as boundary and initial conditions.
This flexibility of learning the solution of the PDE operator, i.e., the solution of the PDE for arbitrary boundary conditions $\ubase$, makes them a viable tool in our method.


To generate training data for the drift-diffusion model on the metric graph we rely on a finite volume implementation described in \cref{sec:FVM} where we fix $\ell_e=1$ for all edges and $T=1$. We solve the PDE \eqref{eq:strong_pde}--\eqref{eq:initial_conditions} on three kinds of graphs depicted in \cref{fig:training_graphs} using this finite volume method (FVM). 
Initial conditions $\uinit$ as well as inflow and outflow conditions $\uinflow_v$ and $\uoutflow_v$ are obtained by sampling from a Gaussian process for all edges $e \in \E$ and all vertices $v \in \V_\D$, resp.
These are evaluated on an equidistant discretization, both in space and time. 
We assume that all random Gaussian processes are approximated through
\begin{equation}
    \label{eq:GPgeneral}
    g(x) = \sum_{k=1}^{n_\text{GP}}\eta_k \,\phi(x-x_k)
\end{equation}
by using a radial basis function (RBF) kernel $\phi(r)=\exp\left({-\norm{r}^2/\ell^2}\right)$ with length scale $\ell = 0.5$, $n_\text{GP} = 512$ equally distributed centers $x_k \in [0,1]$ and $\eta_k \sim \mathcal{N}(0,1)$ normally distributed.

To accommodate for discontinuities in the initial conditions of the randomly initialized graph, we let the finite volume scheme run for a small time-interval and then take the solution at this time as the initial solution for the training of our model at the sensor locations $\xinit$ to obtain $\uinit_e$ along each edge.
The training flux sensor measurements $\uorigin_{e}$ and $\utarget_{e}$ are taken similarly at sensor locations $\torigin$ and $\ttarget$, resp., by evaluation of the flux boundary condition \eqref{eq:Dirichlet_conditions} if either the origin or target vertex belong to $\V_\D$, and by evaluation of the \KN\ condition \eqref{eq:Kirchhoff_Neumann_condition} for vertices in $\V_\K$.
 
 	\begin{figure}
		\begin{center}
			\includegraphics[scale=0.5]{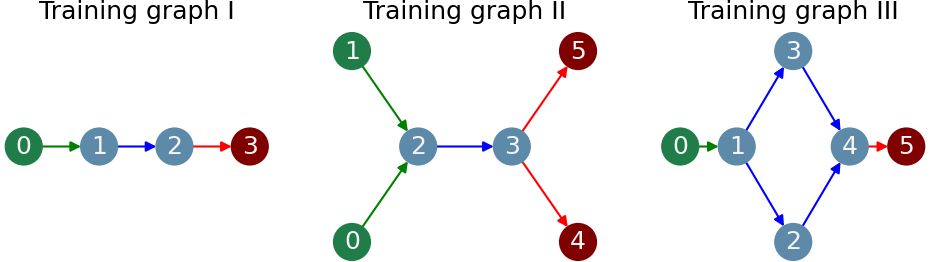}
		\end{center}
		\caption{Model graphs that were used to generate training data for physics-informed \deeponets. Green edges are used to train inflow model, blue ones for inner model and red ones for outflow model.}
  \label{fig:training_graphs}
	\end{figure}

To learn the parameters of our model, we minimize the objective
\begin{multline}
\label{eq:loss_learning}
\sum_{i=1}^{\npdebatch} \L_\text{PDE} (\usensor_i, t_i, x_i)
+ 
\sum_{i=1}^{\ninitbatch} \L_\text{init} (\usensor_i, x_i)\\
+ 
\sum_{i=1}^{\nbcbatch} \L_\text{edge} (\usensor_i, t_i)
\end{multline}
where $\theta$ is the set of trainable parameters of our model, $\npdebatch$, $\ninitbatch$ and $\nbcbatch$ are the respective batch sizes.
By setting $G_\theta^{\usensor}(t, x) := G_\theta(\usensor, t, x)$ as the output of the inflow (resp.\ inner and outflow) operator network, the pointwise PDE loss is defined as
\begin{equation*}
     \L_\text{PDE} (\usensor_i; t_i, x_i) :=\left(\mathcal{H}(G_\theta^{\usensor_i}(t_i,x_i))\right)^2,
     %
\end{equation*}
where the \deeponet\ operator is learned to satisfy the physics model.
For the initial data we use 
\begin{equation*}
    \L_\text{init}(\usensor_i; x_i) = \left( G_\theta^{\usensor_i}(0,x_i) - u_e^{\text{init}}(x_i)\right)^2.
\end{equation*}
The edge loss is the only term which differs among the three different edge types.
We train the inflow model based on 
\begin{multline*}
    \L_\text{edge}^\text{inflow}(\usensor_i; t_i) =\\
    \left( \uorigin_i (t_i) \,(1 - G_\theta^{\usensor_i}(t_i,0)) - J_e(G_\theta^{\usensor_i}(t_i,0))\right)^2\\
    +\left(\utarget_i(t_i) - J_e(G_\theta^{\usensor_i}(t_i,1))) \right)^2,
\end{multline*}
while the loss for the inner model is given by
\begin{multline*}
    \L_\text{edge}^\text{inner}(\usensor_i; t_i) =\\
    \left(\uorigin_i(t_i) - J_e(G_\theta^{\usensor_i}(t_i,0))) \right)^2\\
    +\left(\utarget_i(t_i) - J_e(G_\theta^{\usensor_i}(t_i,1))) \right)^2.
\end{multline*}
Similar to the inflow edge loss, the corresponding outflow edge loss term reads
\begin{multline*}
    \L_\text{edge}^\text{inflow}(\usensor_i; t_i) =\\
    \left(\uorigin_i(t_i) - J_e(G_\theta^{\usensor_i}(t_i,0))) \right)^2\\
    \left( \utarget_i (t_i) \, G_\theta^{\usensor_i}(t_i,1) - J_e(G_\theta^{\usensor_i}(t_i,1))\right)^2.
\end{multline*}
In contrast to the \emph{default} \deeponet\ approach, which is trained using a large number of point evaluations of the solution obtained using some reference  numerical method, the physics-informed approach only relies on the physical model in arbitrary points as well as a set of reasonable initial and boundary measurements.

 The model architecture in the approximation of the operator net follows \cite{wang2021learning}.
 In particular, we use a modified multilayer perceptron (MLP) as branch net and a Fourier network with 5 random frequencies as trunk net.
 We train our models with seven hidden layers and hyperbolic tangent activation function.
 Training of all models is conducted with a gradient clipped Adam optimizer \cite{kingma2014adam}  with an exponentially decaying learning rate schedule and \num{20000} epochs.
 To investigate the influence of the expressivity of the networks, we train a small network with seven hidden layers and width 100 and a large one with hidden dimension 200 for the various edge types.
For each model, we use single precision on a single NVIDIA A40 GPU with three different sets of training data: the $5$K, $10$K and $20$K models use data generated from \num{5000}, \num{10000} and \num{20000} FVM solves with random measure data, resp.
 Since our training graphs depicted in \cref{fig:training_graphs} contain in total $6$ inner edges, $4$ inflow and $4$ outflow edges, the inner model is trained with 50 percent more data than the inflow and outflow edge model.
 We decided to keep this slight imbalance, due to the fact that inner edges appear much more often than boundary edges, especially in larger graphs.
 
 We report the validation loss for each model in \cref{tab:validation_loss_training}.
 One can clearly see that the approximation quality of our model improves significantly if more training data are used.
 Furthermore, the larger model with width \num{200} benefits even more from a larger training set and halves the validation loss when compared to the smaller network. 
 Convergence plots of the various loss terms can be found in \cref{sec:loss_plots}.

 	\begin{table}
    \centering
		\begin{tabular}{crccc}
			\toprule  
			Width&Data& Inflow& Inner & Outflow  \\
			\midrule
			\multirow{ 3}{*}{100} & 5K & 4.34e-03 & 1.42e-03 & 3.63e-03 \\
			& 10K& 1.62e-03 & 8.29e-04 & 1.95e-03\\
			& 20K& 1.09e-03 & 5.67e-04 & 1.05e-03\\
			\midrule
			\multirow{ 3}{*}{200} & 5K & 8.09e-03 &1.42e-03 & 6.07e-03 \\
			& 10K&  1.96e-03  &5.50e-04 & 2.12e-03\\   
			& 20K& 6.64e-04& 2.62e-04 & 6.30e-04\\
            \bottomrule
		\end{tabular}
        \caption{Final validation loss after \num{20000} epochs of training.}
        \label{tab:validation_loss_training}
	\end{table}

	\begin{figure}
		\begin{center}
			\includegraphics[scale=0.25]{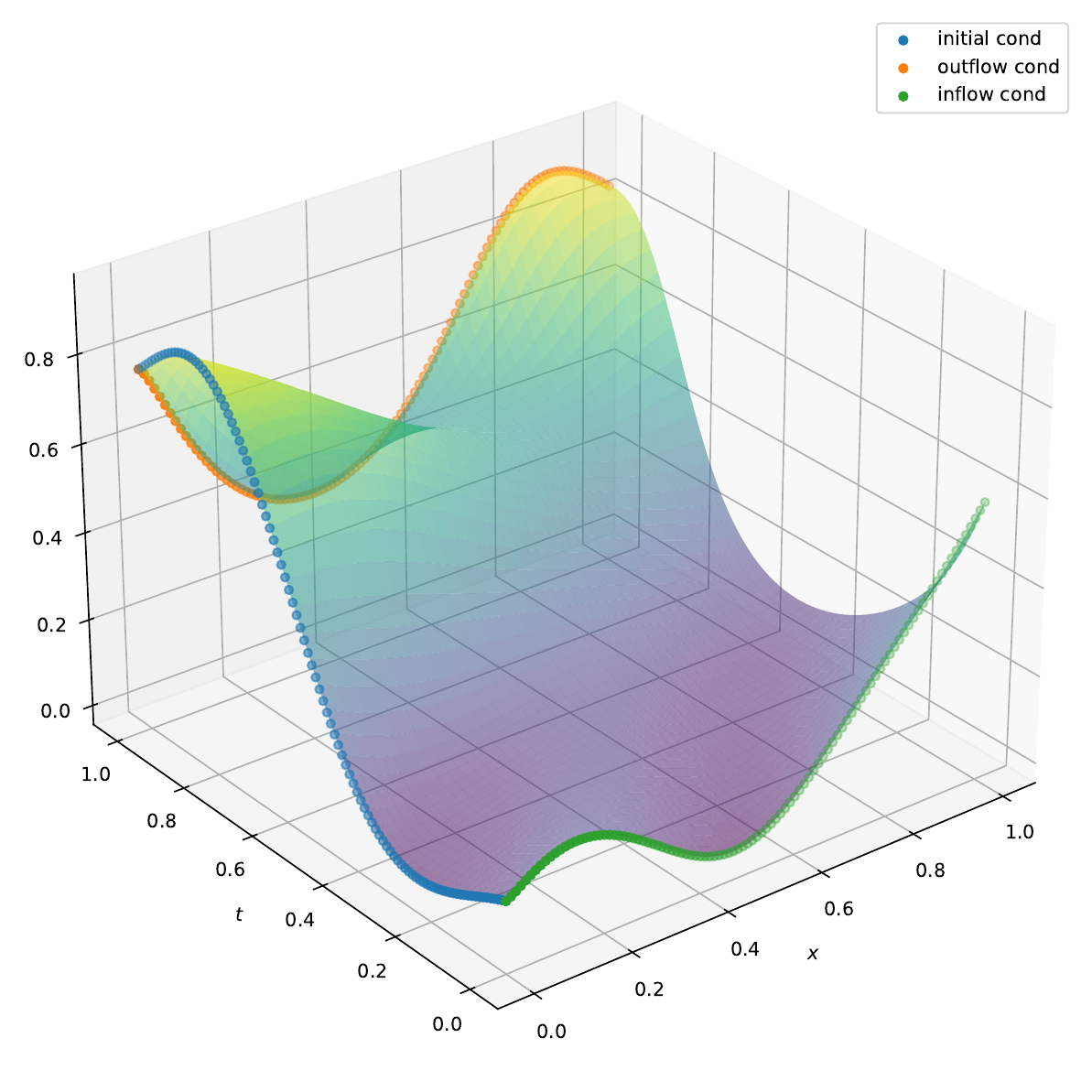}
		\end{center}
		\caption{Illustration of random GP training data: initial condition measurements $\uinit$ (blue), inflow measurements $\uinflow_{v}$ (green), outflow measurements $\uoutflow_{v}$ (orange).}
	\end{figure}
	\begin{figure}
		\begin{center}
			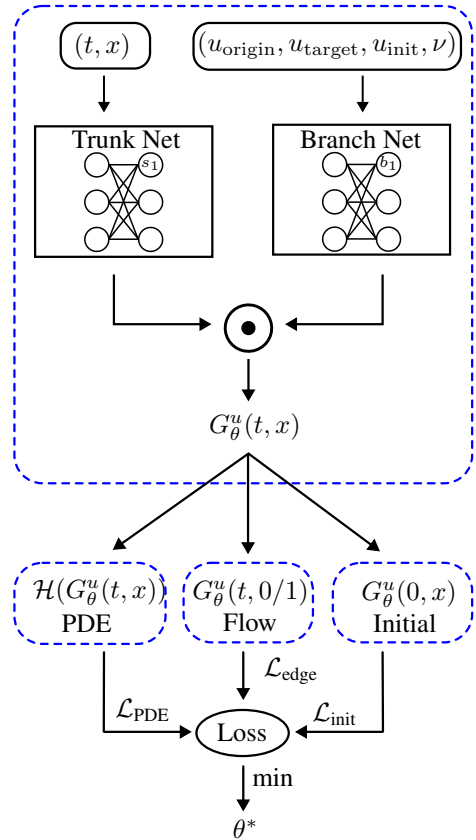 
		\end{center}
		\caption{Illustration of physics-informed \deeponet\ adapted to our setting from \cite{wang2021learning}.}
	\end{figure}


	\section{Model evaluation}\label{sec:moodel_evaluation}

After performing the training procedure discussed in the previous section, we obtain three \deeponet models with different sets of parameters for inflow, outflow and inner edges, resp. We denote them by $G_{\theta_{\text{inflow}}}$, $G_{\theta_{\text{inner}}}$ and $G_{\theta_{\text{outflow}}}$, where we always use letter $G$, as the architecture is the same is all three cases. Each operator network $G_\theta$, for $\theta = \theta_{\text{inflow}},\, \theta_{\text{inner}},\, \theta_{\text{outflow}}$ returns function evaluations of its solution in arbitrary points $(t, x)$ for arbitrary feasible boundary and initial conditions, separately for each edge.

Pursuing our goal to solve the drift-diffusion equation defined in~\eqref{eq:strong_pde} on a complete graph, it remains to obtain the correct input parameters on each edge such that the continuity and Kirchhoff-Neumann conditions given in \eqref{eq:continuity_condition} and \eqref{eq:Kirchhoff_Neumann_condition} resp., are satisfied at each vertex.
	Learning these unknown flow parameters $\flow$ is done by minimization of the loss function 
 \begin{equation}\label{eq:def_loss_learning}
     \mathcal{L}_{\text{coupling}} (\flow) = \sum_{i} \mathcal{L}_{\text{c}}(t_i,\flow)
 \end{equation}
 where $\mathcal{L}_{\text{c}}(t_i,\flow)$ is defined by
	\begin{multline*}
        \underbrace{%
        \frac{1}{\abs{\V_\K}}\sum_{v \in \V_\K} \frac{1}{\abs{\E_v}} \sum_{e, e' \in \E_v} (\hat\rho^{u(\flow)}_e(t_i,v) - \hat\rho^{u(\flow)}_{e'}(t_i,v) )^2}_{\text{continuity loss}}\\
		+
		\underbrace{\frac{1}{\abs{\V_\K}} \sum_{v \in \V_\K} \frac{1}{\abs{\E_v}} \Big(\sum_{e \in \E_v} (\hat{J}^{u(\flow)}_e(t_i,v) \, n_e(v) \Big)^2}_{\text{Kirchhoff loss}}
	\end{multline*}
	where the first term ensures the continuity of the flow values at each node for all edges.The second Kirchhoff term ensures that the conservation of mass across the overall graph and all nodes.
 Here, the value $\hat\rho^{u(\flow)}_e(t_i,v)$ corresponds to the evaluation of the \deeponet\ for $u(\flow)$ depending on the respective edge type of $e$ at time $t_i$.
 Similarly, $\hat{J}^{u(\flow)}_e(t_i,v)$ represents the evaluation of the flux.
 We here assume that the vector $\flow$ is approximated by a kernel interpolation using a radial basis function (RBF) kernel with fixed parameters resulting in 
 $$
 \flow(t)=\sum_{k=1}^{n_{\beta}}\beta_k\phi(t-t_k)
 $$
 where $n_{\beta}=10$ is chosen to further reduce the computational complexity and $t_k$ are uniformly distributed in $[0,1].$ The kernel function $\phi(r)=\exp\left({\frac{-\norm{r}^2}{\ell^2}}\right)$ with $\ell=0.2$. To illustrate this in a bit more detail consider the following inflow edge modeled via 
$$
(\uorigin, \utarget, \uinit)\in\R^{n_{\textrm{origin}}+n_{\textrm{target}}+n_{\textrm{init}}}.
$$
where the values for inflow, stored in $\uorigin$, and initial condition $\uinit$ are known and the values for the outflow, encoded in $\utarget$, have to be determined. Since $\utarget$ is parameterized using the above-described RBF interpolation we now learn the parameters $\beta$ for the outflow condition. To address this challenge on the whole graph we learn the values for $\beta$ at all nodes to enforce the PDE, the coupling and continuity conditions as well as initial and boundary (inflow plus outflow) conditions. The parametrization of this system only requires $2n_{\beta}$ parameters for all inner edges and $n_\beta$ parameters for inflow and outflow edges. With $n_\beta$ small the resulting learning can be done in a matter of seconds using a standard gradient based optimization algorithm such as Adam implemented in \jax\ \cite{jax2018github}.

 	\begin{figure}
		\begin{center}
			\includegraphics[scale=0.5]{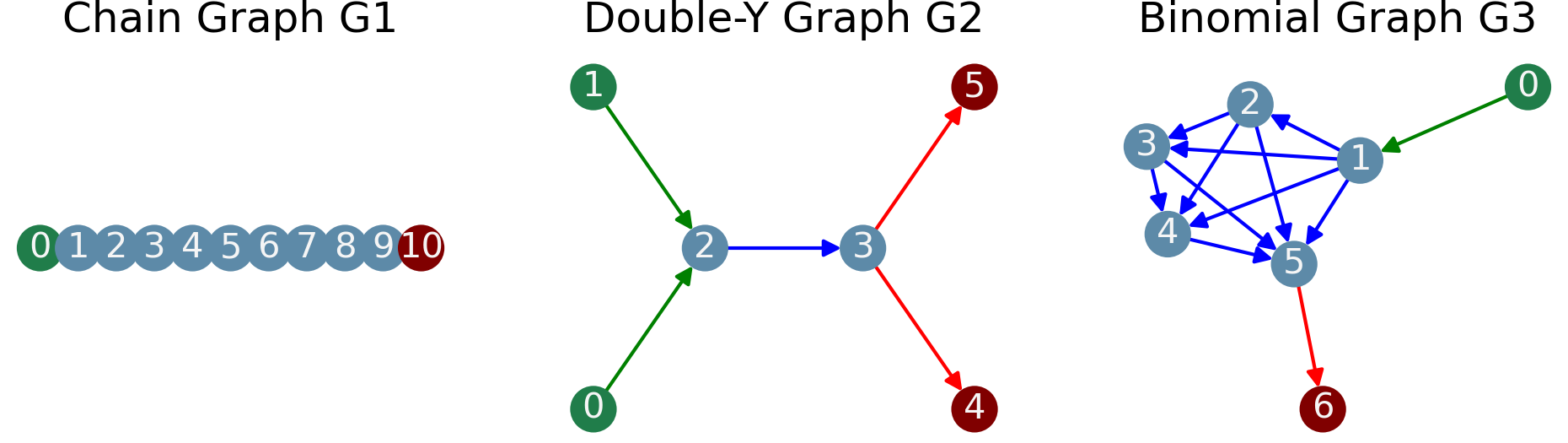}
		\end{center}
		\caption{Model graphs that were used to verify our methodology.}
  \label{fig:test_graphs}
	\end{figure}

The results confirm that our methodology is able to learn the solution of the drift-diffusion PDE on graphs.
In the upper part of \cref{fig:double_y_inference} we plot both the physics-informed \deeponet\ and the reference solution.
The error terms shown below indicate that the approximation error is small, see also \cref{tab:l2_abs_error} and \cref{tab:l2_rel_error} for detailed values.
\Cref{fig:unrolled_chain_inference} shows that our method is able to capture nonsmooth transitions at the vertices of a chain graph.
Again, the solution of the \deeponet\ and the reference solution are visually indistinguishable.

	\begin{figure}
	    \centering
	    \includegraphics[width=1\linewidth]{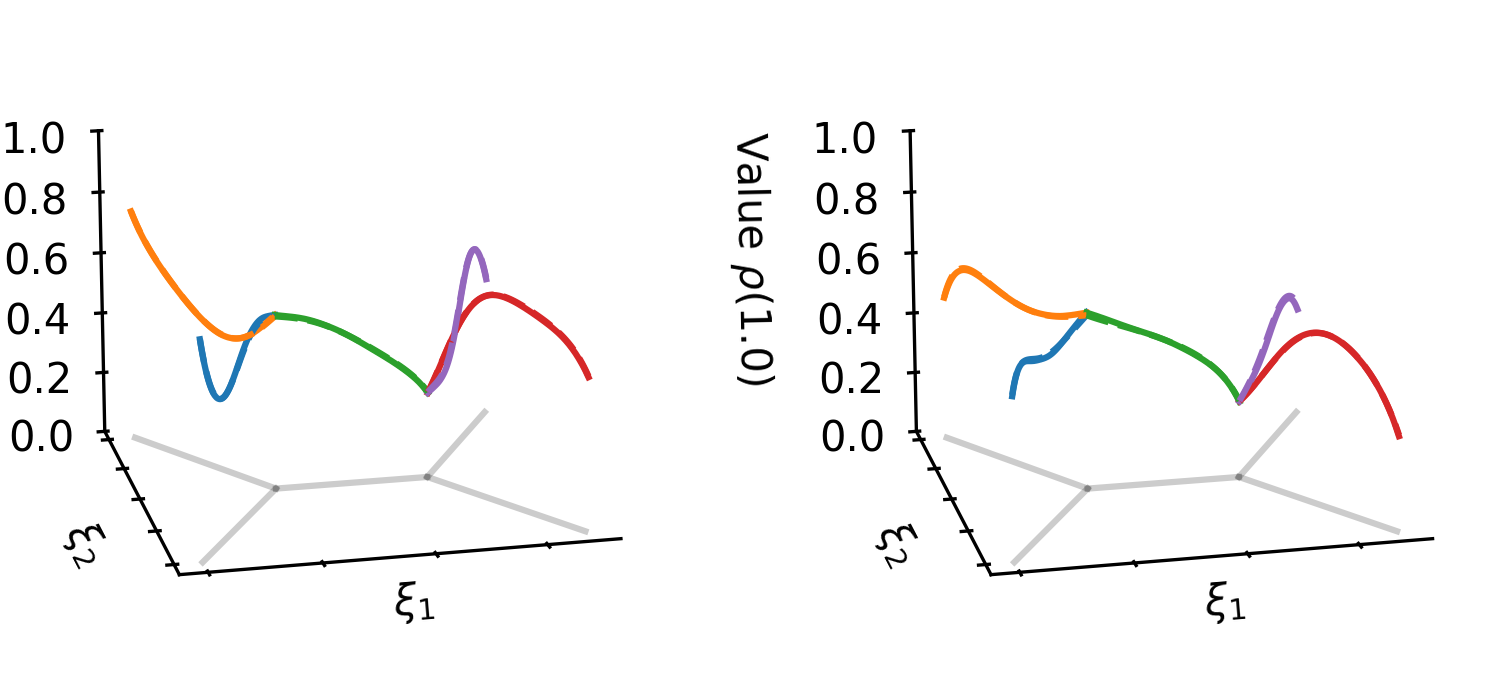}\\
	    \includegraphics[width=1\linewidth]{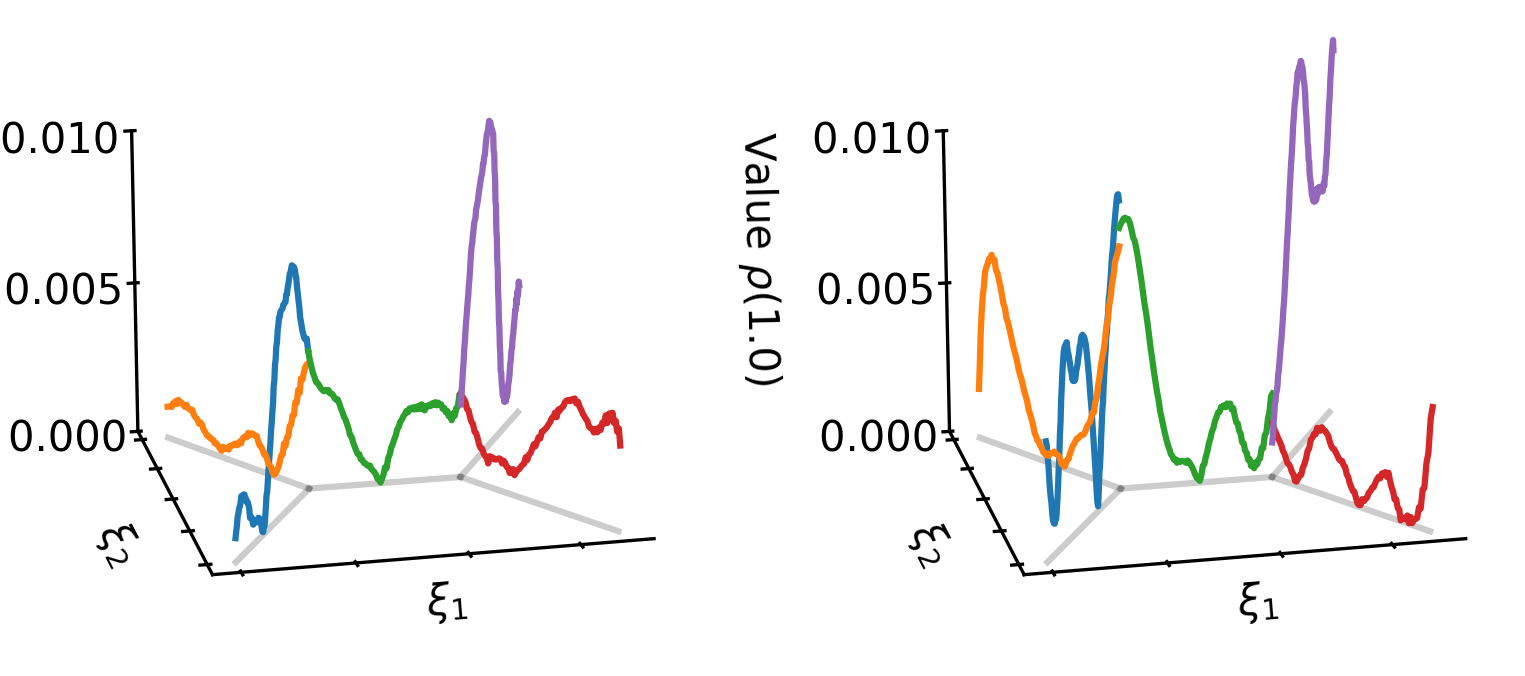}
	    \caption{Upper row: Almost indistinguishable reference solution (solid) and PI \deeponet\ solution (dashed) on model graph at $t=0.5$ (left) and $t=1.0$ (right). Lower row: Absolute difference between reference and PI \deeponet\ solution.}
	    \label{fig:double_y_inference}
	\end{figure}
 
	\begin{figure}
	    \centering
	    \includegraphics[width=.8\linewidth]{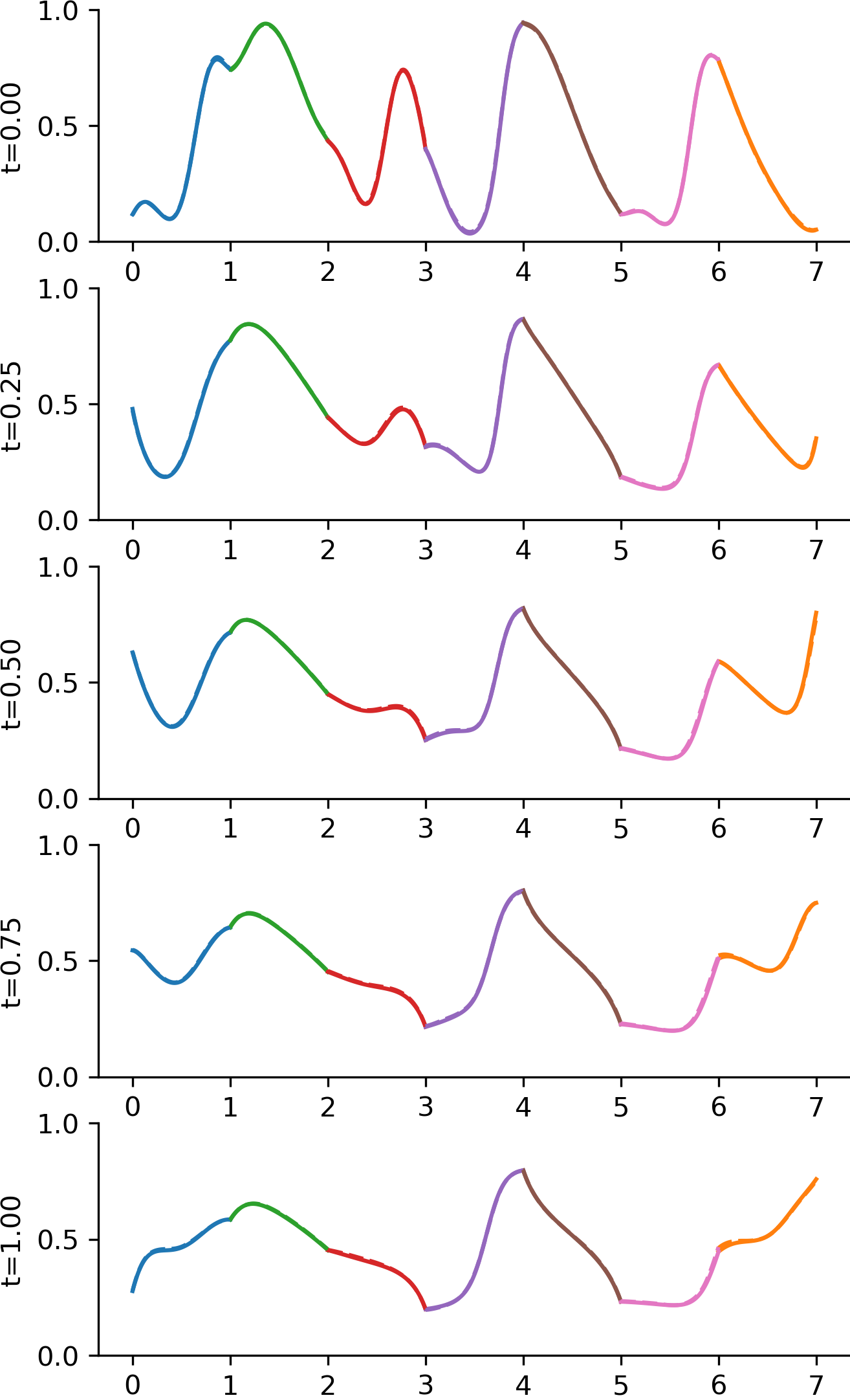}
	    \caption{Reference solution (solid) and PI \deeponet\ solution (dashed) on unrolled chain graph with 7~edges over time.}
	    \label{fig:unrolled_chain_inference}
	\end{figure}
 	 	\begin{table}
    \centering
		\begin{tabular}{crccc}
			\toprule  
			Width&Data& $G_1$ & $G_2$ & $G_3$  \\
			\midrule
			\multirow{ 3}{*}{100} & 5K & 5.50e-02 & 3.27e-02 & 3.97e-02 \\
			& 10K& 9.38e-03 & 1.29e-02 & 1.31e-02\\
			& 20K& 8.46e-03 & 1.03e-02  & 1.11e-02 \\
			\midrule
			\multirow{ 3}{*}{200} & 5K & 2.87e-02 & 1.71e-02 & 2.35e-02 \\
			& 10K&  6.06e-03  & 7.66e-03 & 7.59e-03\\   
			& 20K& 4.68e-03 & 5.62e-03 & 5.81e-03 \\
            \bottomrule
		\end{tabular}
        \caption{Absolute space-time $L^2$-error between solution of the FVM code compared to the output of \deeponet\ averaged over \num{1000} runs with randomly drawn initial and boundary conditions. These are sampled as Gaussian processes  \eqref{eq:GPgeneral} with $n_\text{GP} = 468$ and $\ell=0.4$.}
        \label{tab:l2_abs_error}
	\end{table}
 
	 	\begin{table}
    \centering
		\begin{tabular}{crccc}
			\toprule  
			Width&Data& $G_1$ & $G_2$ & $G_3$  \\
			\midrule
			\multirow{ 3}{*}{100} & 5K & 1.28e-01 &  7.05e-02 & 9.74e-02 \\
			& 10K& 2.15e-02 & 2.73e-02 & 3.21e-02\\
			& 20K& 1.96e-02 &2.21e-02  & 2.59e-02\\
			\midrule
			\multirow{ 3}{*}{200} & 5K & 5.48e-02 & 3.45e-02 & 4.79e-02 \\
			& 10K& 1.37e-02 & 1.62e-02 & 1.73e-02\\   
			& 20K& 1.06e-02 & 1.20e-02 & 1.30e-02\\
            \bottomrule
		\end{tabular}
        \caption{Relative space-time $L^2$ error between solution of the FVM code compared to the output of \deeponet\ averaged over \num{1000} runs with randomly drawn initial and boundary conditions. These are sampled as Gaussian processes \eqref{eq:GPgeneral} with $n_\text{GP} = 468$ and $\ell=0.4$.}
        \label{tab:l2_rel_error}
	\end{table}

	\section{Inverse problems}
	The methodology developed in the previous sections is especially suited for the efficient solution of large scale parameter identification problems on graphs, as this amounts to merely add data misfit terms to \eqref{eq:def_loss_learning}.

    As a toy application, we think of a traffic network with measurement sensors located at the midpoint of each edge.
    We assume that they are able to measure both the density and the flux of vehicles at their respective location and as a function of time using modern sensor hardware and corresponding traffic flow estimation algorithms, see \cite{Seo2017_trafficflowestimation} for more details. We denote these time discrete measurements by 
    $\rho_e^{\text{meas}}\in \R^{n_{\text{meas}}}$ and $j_e\in \R^{n_{\text{meas}}}$, and by $x_e^{\text{meas}}$ the location of the sensor on each edge. To add this information to our model, we now simply extend $\mathcal{L}_{\text{coupling}}$ by the following additional loss terms
      \begin{multline*}
    \quad \frac{1}{n_\text{meas}} \sum_{i=1}^{n_\text{meas}}
    \frac{1}{\abs{\E_v}}
\sum_{e \in \E_v}\left[ (\hat\rho^{u(\flow)}_e(x_{e,i}^{\text{meas}}) -  \rho_{e,i}^{\text{meas}})^2 \right.\\
\left. +		  (\hat{J}^{u(\flow)}_e(x_{e,i}^{\text{meas}}) - j_{e,i}^{\text{meas}})^2\right]. \quad   
    \end{multline*}
    Thus, the algorithm explained in Section~\ref{sec:moodel_evaluation} can be used to tackle the inverse problem without any major changes. 
    After completing the optimization procedure, we automatically solved several inverse problems: Evaluating the vector $u$, we recover the unknown initial condition and also the velocities $\nu_e$ on each edge. What is more, evaluating $\hat\rho_e^u(t,x)$ at any time in the simulation interval, we also obtain access to the densities on the complete graph without the need to perform another forward simulation of the model. 

    We test our methodology on the three test graphs depicted in \cref{fig:test_graphs} where we choose $n_{\text{meas}}=101$. 
    For illustration, the learned unknown initial conditions and velocities as well as the inferred solutions of a chain graph with seven edges are depicted in \cref{fig:unrolled_chain_inverse} for various levels of additive measurement noise $\epsilon = 0.1, 0.05, 0.01$. We observe that we are able to recover all the essential features of the initial conditions- but also of the dynamics at later times. In particular, due to the fitting of space-time data, the error remains roughly constant in time, at least in the eye ball norm. As for the prediction of the velocity, we observe that the accuracy of certain edges away from the ends of the chain have a substantially larger error for high noise levels than others, which we will investigate further in future works. Nevertheless, the example shows that our approach is feasible for the parameter identification problem and even suitable for possible real-time applications such as traffic flow prediction.

    A more systematic error analysis can we found in  \cref{tab:inverse_abs_errors} and \cref{tab:inverse_rel_errors}, where we report the parameter identification capability of our method using the large model (width 200, 20K training data) by using a space-time $L^2$ error measure.
    
	\begin{figure}
	    \centering
	    \includegraphics[width=.8\linewidth]{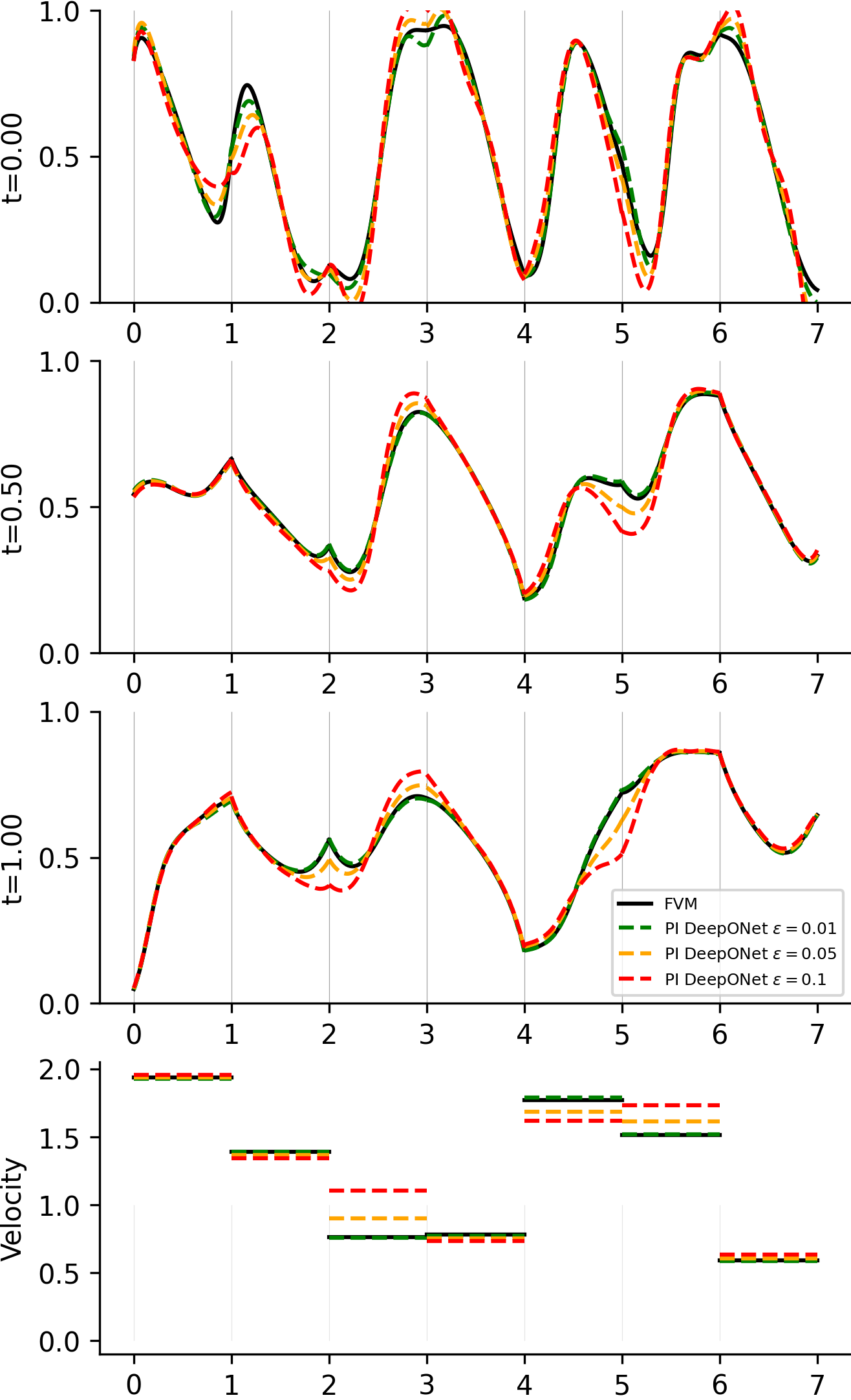}
	    \caption{Illustration of parameter identification on chain graph with 7~edges. Reference solution (solid) and physics-informed \deeponet\ solutions (dashed) for noise levels $\epsilon=0.1, 0.05, 0.01$. Bottom row depicts recovered edge velocities, first row corresponds to recovered initial conditions.}
	    \label{fig:unrolled_chain_inverse}
	\end{figure}
    
	\begin{table}
    \centering
		\begin{tabular}{ccccc}
			\toprule
			&& err. init & err. vel. & $\| \rho - \rho_\text{ref} \|_{L^2}$ \\
			\midrule
			\multirow{ 3}{*}{$G_1$}&$\epsilon_1$ & 9.13e-02 & 9.39e-02 & 4.42e-02 \\
			&$\epsilon_2$ & 5.79e-02 & 5.32e-02 & 2.76e-02 \\
			&$\epsilon_3$ & 4.02e-02 & 3.05e-02 & 1.83e-02 \\
			\midrule
			\multirow{ 3}{*}{$G_2$}&$\epsilon_1$ & 1.86e-01 & 9.91e-02 & 6.73e-02 \\
			&$\epsilon_2$ & 1.01e-01 & 5.41e-02 & 3.79e-02 \\
			&$\epsilon_3$ & 4.68e-02 & 1.98e-02 & 1.65e-02 \\
			\midrule
			\multirow{ 3}{*}{$G_3$}&$\epsilon_1$ & 1.08e-01 & 9.52e-02 & 4.27e-02 \\
			&$\epsilon_2$  & 6.76e-02 & 5.53e-02 & 2.54e-02 \\
			&$\epsilon_3$ & 5.03e-02 & 3.32e-02 & 1.70e-02\\
			\bottomrule
		\end{tabular}
		\caption{Absolute $L^2$-errors for parameter identification problem on test graphs depicted in \cref{fig:test_graphs} with measurement noise $\epsilon_1 = 0.1$, $\epsilon_2=0.05$, $\epsilon_3=0.01$ averaged over 100~runs.}
  \label{tab:inverse_abs_errors}
	\end{table}

	\section{Conclusion}
 We provide a novel physics-informed \deeponet\ architecture for creating a surrogate model that allows the efficient solution of a drift diffusion model on a possibly complex metric graph. Additionally, our model allows to solve an inverse problem for this setup at almost no additional costs. The flexibility of traing \deeponet\ submodels for inflow, outflow, and inner edges allows the construction of drift diffusion models (and in a similar fashion other PDEs) on complex graphs in a Lego-like way with linear complexity in the number of edges. This would allow straightforwardly the application to real traffic data, which is readily available in several open databases (e.g. \cite{Loder2019}). Adding the respective graph topology is no obstacle, while more complex traffic equations beyond the  drift-diffusion model would require the adjustment of the physics loss in our suggested approach, see \cite{piccoli2006traffic} for an overview of such models.
	\begin{table}
    \centering
		\begin{tabular}{ccccc}
			\toprule
			&& err. init & err. vel. & $\| \rho - \rho_\text{ref} \|_{L^2}$ \\
			\midrule
			\multirow{ 3}{*}{$G_1$}&$\epsilon_1$ & 1.85e-01 & 8.63e-02 & 9.49e-02\\
			&$\epsilon_2$ & 1.19e-01 & 4.67e-02 & 5.97e-02\\
			&$\epsilon_3$ & 8.03e-02 & 2.29e-02 & 3.80e-02\\
			\midrule
			\multirow{ 3}{*}{$G_2$}&$\epsilon_1$ & 3.93e-01 & 9.19e-02 & 1.50e-01 \\
			&$\epsilon_2$ & 2.12e-01 & 4.82e-02 & 8.36e-02\\
			&$\epsilon_3$ & 9.25e-02 & 1.68e-02 & 3.35e-02 \\
			\midrule
			\multirow{ 3}{*}{$G_3$}&$\epsilon_1$ & 2.19e-01 & 8.10e-02 & 8.85e-02 \\
			&$\epsilon_2$ & 1.37e-01 & 4.47e-02 & 5.30e-02 \\
			&$\epsilon_3$ & 1.04e-01 & 2.57e-02 & 3.62e-02 \\
			\bottomrule
		\end{tabular}
		\caption{Relative $L^2$-errors for parameter identification problem on test graphs depicted in \cref{fig:test_graphs}  with measurement noise $\epsilon_1 = 0.1$, $\epsilon_2=0.05$, $\epsilon_3=0.01$ averaged over 100~runs.}
  \label{tab:inverse_rel_errors}
	\end{table} 
	\section*{Software and Data}

        If the paper is accepted, we publish all software and data that is necessary to reproduce the results on GitHub.
	
	
	
	
	\section*{Impact Statement}
	
This paper presents work whose goal is to advance the field
of Machine Learning. There are many potential societal consequences of our work, none of  which we feel must be specifically highlighted here.

	\newpage
	\appendix
	\onecolumn
    \section{Proof of Theorem \ref{thm:existence}}\label{sec:proof_existence}
    The proof is based on extending ideas from \cite{Gomes2019, Burger2020}, where in- and outflow boundary conditions are treated, to metric graphs.

    We will work with Sobolev spaces defined on the metric graph $\mathcal{G}=(\V, \E, (l_e)_{e\in\E})$. We first introduce the space of square integrable functions
    \begin{align*}
        L^2(\E) := \{ v: \E \to \R \; : \forall e \in \E\;v_e = v|_e\in L^2(e) = L^2(0,l_e) \}.
    \end{align*}
    It is a Hilbert space with scalar product
    \begin{align*}
        \langle v, w\rangle_{\E}=\sum_{e \in \E}\left\langle v_e, w_e\right\rangle_e=\sum_{e \in \E} \int_0^{\ell_e} v_e w_e \;dx,
    \end{align*}
    which induces the norm $\|u\|_{L^2(\E)}= \sqrt{\langle u,u\rangle}$. Furthermore, the space $H^1(\E)$ of functions having a square integrable weak derivative is then defined by
    $$
    H^1(\mathcal{E})=\left\{w \in L^2(\mathcal{E}): \partial_x w_e \in L^2(e) \text { and } w_e(v)=w_{e^{\prime}}(v) \quad \forall e, e^{\prime} \in \mathcal{E}(v), v \in \mathcal{V}_\K\right\}.
    $$
    We further denote by $(H^1(\E))'$ the dual space, i.e. the space containing all linear, bounded functionals on $H^1(\E)$. 
    
    Space-time dependent functions are considered as time-dependent functions with values in a function space, say $v(t) \in X$. For such functions, we introduce the norm $\|v\|_{L^2(0,T;X)} = \int_0^T \|v(t)\|_X \;dt$ and use the notation $L^2(0,T;X)$ for all functions such that this norm is finite. In complete analogy, we also define $H^1(0,T;X)$ and the energy space
    \begin{align}\label{eq:energy_space}
        W(0,T) = L^2(0,T;H^1(\E)) \cap H^1(0,T;(H^1(\E))').
    \end{align}
    The actual proof of Theorem~\ref{thm:existence} is based on the use of the formal gradient flow structure of the problem, i.e. the fact that the entropy functional 
    \begin{align}
        E(\rho) = \sum_{e\in \E} \int_0^{l_e} \eps(\rho_e\log(\rho_e) + (1-\rho_e)\log(1-\rho_e)) + \rho\,\nu_e\,x \;dx,
    \end{align}
    is a Lyapunov functional. 
    For readability, we used $f(\rho_e) = \rho_e(1-\rho_e)$ in this definition, noting that the proof works completely analogous for other choices of $f$, upon modifying the entropy functional.
    
    Based on this, we introduce the entropy variable $w$ defined as the variational derivative of the entropy, i.e.
\begin{equation}
\begin{aligned}
    \label{eq:defn-entr-var-w}
    w(\rho)&\coloneqq \frac{\delta E}{\delta \rho}= \sum_{e\in\E}\eps(\log(\rho_e)-\log(1-\rho_e)) + \nu_e x.
\end{aligned}
\end{equation}
such that the transformation from $w$ to $\rho$ is given by
\begin{equation*}
    \rho_e=\frac{e^{\frac{w_e-\nu_e\,x}{\eps}}}{1+e^{\frac{w_e-\nu_e\,x}{\eps}}},\quad \forall e \in \E,
    \label{eq:transform_w_to_u}
\end{equation*}
where $w_e = w|_e$. 

The proof is based on a time-discretization and regularization strategy using these variables. To this end, let $N\in\mathbb{N}$ be such that $(0,T]$ has sub-intervals of the form 
\begin{equation*}
    (0,T]=\bigcup_{k=1}^N \big((k-1)\tau, k\tau\big],
\end{equation*}
where $\tau=\frac{T}{N}$ and $t_k = \tau k$. 
Using $\eps\,\partial_x \rho_e + \rho_e\,(1-\rho_e)\,\nu_e = \rho_e\,(1-\rho_e)\,\partial_x w_e$ for all $e\in\E$, we introduce the time-discretized and regularized form of \eqref{eq:weak} with $w$ as unknown as
\begin{multline}\label{eq:weak_ent}
\sum_{e \in \E} \int_e  \left(\frac{\rho_e^{k}-\rho_e^{k-1}}{\tau}\,\varphi_e +
\rho_e^{k}\,(1-\rho_e^{k})\,\partial_x w_e^{k}\,\partial_x \varphi_e \right)dx 
+ \tau \sum_{e \in \E} \int_e \left(\partial_x w_e^{k}\,\partial_x \varphi_e + w_e^{k}\,\varphi_e\right)dx \\
+\sum_{v \in \V_D} \left(-\uinflow_{v}(t_{k})\, (1-\rho^{k}(v)) + \uoutflow_v(t_{k}) \,\rho^{k+1}(v)\right)\,\varphi(v) = 0,
\end{multline}
for all test functions $\varphi \in H^1(\Gamma)$ and given nonnegative functions $\uinflow_{v}, \uoutflow_v \in L^\infty(0,T)$, $v\in \V_D$. In addition, we added a regularization term, multiplied by $\tau$.

Our first aim is to show existence of iterates $\rho^k$ satisfying the (still non-linear) equation \eqref{eq:weak_ent}. This is done using a linearisation strategy and a fixed-point argument. We begin by defining the set 
$$
    \A:= \left\{ \rho \in L^\infty(\E) :  \, 0 \le \rho_e  \le 1, \, e \in \E \right\}.
$$
\begin{lemma}\label{lem:s1}
Given $\tau > 0$ and  $\tilde \rho \in \A$, the linear problem 
\begin{multline}\label{eq:discrete_linearized}
\sum_{e \in \E} \int_e  \left(\frac{\tilde \rho_e-\rho_e^{k-1}}{\tau}\,\varphi_e +
\tilde\rho_e\,(1-\tilde\rho_e)\,\partial_x w_e\,\partial_x \varphi_e \right)dx 
+ \tau \sum_{e \in \E} \int_e \left(\partial_x w_e\,\partial_x \varphi_e + w_e\,\varphi_e\right)dx \\
+\sum_{v \in \V_D} (-\uinflow_{v}(t_{k+1})\,(1-\tilde\rho(v)) + \uoutflow_v(t_{k+1})\,\tilde\rho(v))\,\varphi(v) = 0,
\end{multline}
for all $\varphi \in H^1(\E)$, has a unique solution $w \in H^1(\E)$ such that
\begin{align}
    \label{eq:discrete_linearized_a_priori}
    \|w\|_{H^1(\E)} \le C,
\end{align}
where the constant $C>0$ depends only on $\tau$ and $\rho^{k-1}$.

In addition, the operator $S_1 : \A \to L^2(\E)$, which assigns some given $\tilde \rho \in \A$ to $w$, being the solution to Eq.~\eqref{eq:discrete_linearized}, is continuous and compact. 
\end{lemma}
\begin{proof}
As $\tilde \rho\,(1-\tilde \rho) \ge 0$ and $\tau > 0$, the existence  of a unique solution $w\in H^1(\E)$ is a direct consequence of the Lax-Milgram Lemma, cf. \cite{Brezis2010}, estimating the boundary terms as in \cite{Burger2016}[Theorem 3.5]. The a priori bound follows by choosing $\varphi = w$ as a test function and applying a trace theorem and the weighted Young inequality to the right-hand side of Eq.~\eqref{eq:discrete_linearized}.

For the continuity of $S_1$, consider a sequence $(\tilde u_n, \tilde m_n) \in \A$ such that $(\tilde u_n, \tilde m_m) \to (\tilde u, \tilde m) \in \A$ and denote by $w_n$ and $w$ the respective solutions to Eq.~\eqref{eq:discrete_linearized}. Subtracting the respective equations and choosing $\varphi = (w-w_n)$ yields 
the convergence $w_n \to w$ in $H^1(\Omega)$. This allows us to pass to the limit in the weak formulation \eqref{eq:discrete_linearized}.

Finally, the compactness of $S_1$ then follows from the compactness of the embedding $H^1(\E) \hookrightarrow L^2(\E)$.
\end{proof}

We are now in a position to the existence of iterates, i.e. solutions to \eqref{eq:weak_ent}.
\begin{theorem}
	\label{thm:fixed-point}
For given $\rho^{k-1} \in \A$ with and $\tau > 0$, there exist a weak solution $\rho^k \in \A \cap H^1(E)$ to Eq.~\eqref{eq:weak_ent}. In addition, for $\tau > 0$, it holds that $0 < \rho^k < 1$.
\end{theorem}
\begin{proof}
    We define the additional operator $S_2: L^2(\E) \to \A$ by means of 
    $$
    [S_2 w]_e = 
    \frac{e^{\frac{w_e-\nu_e\,x}{\eps}}}{1+e^{\frac{w_e-\nu_e\,x}{\eps}}}.
    $$
    Using that $S_2$ is clearly continuous as well as the results of Lemma~\ref{lem:s1}, the operator 
    $$
    S = S_2 \circ S_1 : \A \to \A
    $$
    is well-defined, continuous and compact. 
    Furthermore, it is readily observed that $\A$ is a convex subset of $L^\infty(\Omega)$. 
    Thus, an application of Schauder's fixed point theorem yields the existence of a fixed point $\rho^k \in \A$ associated to $w^k = S_1(\rho^k)$. By definition of $S_2$, the $H^1$-regularity of $w^k$ directly implies $\rho^k \in H^1(\E)$ as well, which allows us to identify the fixed point as a weak solution of Eqs.~\eqref{eq:weak_ent}.

    Finally, we note that as the edges are one-dimensional domains, we also have the embedding $H^1(\E) \hookrightarrow L^\infty(\E)$. Appealing again to the definition of $S_2$, this implies the strict bounds $0<\rho^k<1$.
\end{proof}
Next, we use the entropy functional to show that the iterates $\rho^k$ are bounded, uniformly in $\tau$. This will then allow to extract converging subsequences whose limit will be the desired solution of the original problem \eqref{eq:weak}.

\begin{proposition}
    \label{prop:discrete-entropy}
    Let $(\rho^k)_{k=0}^\infty \subset \A \cap H^1(\E)$ be solutions to \eqref{eq:weak_ent}.  Then, for any $k \in \mathbb N$, the following discrete entropy estimate holds
    \begin{align}
        \label{eq:entropy-dissipation-discrete}
        &\frac1\tau\left(E(\rho^k) - E(\rho^{k-1}) \right)
        + \tau \sum_{e\in\E}\int_0^{l_e} |\partial_x w^k_e|^2 + |w^k_e|^2 \,dx   + \sum_{e\in\E}\int_\Omega \rho^k(1-\rho^k) |\partial_x w^k_e|^2 \le 0.
    \end{align} 
    Moreover,
    \begin{align}
        \label{eq:grad-u-by-grad-m}
        \|\partial_x\rho^k\|_{L^2(\E)}^2 &\leq C +   \frac1\tau \prt*{E(\rho^{k-1}) - E(\rho^k)}, \\ \label{eq:sqrt_tau_w_est}
        \tau^{1/2} \|w_\tau\|_{L^2(0,T; H^1(\E))}& \leq C,
    \end{align}
    where $w_\tau$ denotes the piecewise constant in time interpolation of $w^k$ and with $C>0$ independent of $\tau$ and $k$.
\end{proposition}
\begin{proof}

Owing to the strict upper and lower bounds provided by Theorem~\ref{thm:fixed-point}, the logarithmic terms appearing in the derivative of the entropy are well defined. Thus, we can use the joint convexity of the energy to obtain
    \begin{align*}
        E(\rho^k) - E(\rho^{k-1}) &\leq \sum_{e\in\E}\int_0^{l_e} \brk*{\eps(\log \rho_e^k - \log(1 -\rho_e^k)) + \nu_e x}(\rho^{k}_e - \rho_e^{k-1}) \; dx 
    \end{align*}
    Using the definition of the entropy variable, Eq.~\eqref{eq:defn-entr-var-w}, this results in
    \begin{align*}
        E(\rho^k) - E(\rho^{k-1}) \leq \sum_{e\in\E}\int_0^{l_e} w^k_e (\rho^k_e - \rho^{k-1}_e )\;  dx.
    \end{align*}
	Due to the $H^1(\E)$-regularity of $w^k$, we may use it as a test functions in Eq.~\eqref{eq:weak_ent} to get
    \begin{align}\label{eq:entropy_entropy_diss_proof}
        \frac1\tau\prt*{E(\rho^k) - E(\rho^{k-1}) }
        &\leq  - \tau \sum_{e\in\E}\int_0^{l_e} \prt*{|\partial_x w_e^k|^2 + |w_e^k|^2} dx\nonumber\\
        &\quad - \sum_{e\in\E}\int_0^{l_e} \rho_e^k(1-\rho_e^k) |\partial_x w_e^k |^2\; dx \nonumber\\
        & \quad -\sum_{v \in \V_D} (-\uinflow_{v}(t_{k+1})\,(1-\rho^k(v)) + \uoutflow_v(t_{k+1})\,\rho^k(v))\,w^k(v),
    \end{align}
    Using the definition of $w^k$, cf. Eq \eqref{eq:defn-entr-var-w}, and the fact that inflow vertices are always located at $x=0$ on each edge, we rewrite the inflow terms as follows
    \begin{align*}
        \begin{aligned}
&-\sum_{v \in \V_D} (-\uinflow_{v}(t_{k+1})\,(1-\rho^k(v))w^k(v)  = -\sum_{v \in \V_D} (-\uinflow_{v}(t_{k+1})\,(1-\rho^k(v))\,\eps\,(\log \rho^k(v)-\log (1-\rho^k(v))) \\
&\quad  =  \sum_{v \in \V_D} \uinflow_{v}(t_{k+1})\left[-(1-\rho^k(v)) \log \frac{1-\rho^k(v)}{\rho^k(v)}+2 \rho^k(v)-1\right] + \sum_{v \in \V_D} \uinflow_{v}(t_{k+1})\left[ - 2\rho^k(v)+1\right]
\le C.
\end{aligned}
    \end{align*}
    We recognize the first term as a relative entropy which is non-positive, while the second term is bounded since $0\le \rho^k \le 1$. 
    A similar argument implies that the outflow terms are bounded as well.
    Finally, using once more the definition of $w_k$, cf. Eq \eqref{eq:defn-entr-var-w}, we further estimate the second term in \eqref{eq:entropy_entropy_diss_proof}, using the weighted Young inequality,    
    \begin{align*}
\sum_{e\in\E}\int_0^{l_e}\int_\Omega  \rho_e^k\,(1-\rho_e^k)\,|\partial_x w_k |^2\; dx 
& \ge \sum_{e\in\E}\left(\int_0^{l_e} \frac{|\partial_x \rho_e^k|^2}{\rho_e^k(1-\rho_e^k)} \;d x
- 2 \int_0^{l_e} |\nu_e|\, |\partial_x \rho_e^k|\;dx + \int_0^{l_e} \rho_e^k\,(1-\rho_e^k)\,\nu_e^2\; d x\right) \\
& \geq \sum_{e\in\E}\int_0^{l_e} \frac{|\partial_x \rho_e^k|^2}{2\,\rho_e^k\,(1-\rho_e^k)}\;dx-\int_0^{l_e}\rho_e^k\,(1-\rho_e^k)\,|\nu_e|^2\;d x \\
& \geq 2 \sum_{e\in\E}\int_0^{l_e}|\partial_x \rho_e^k|^2\; d x-\frac{1}{4}\,\sum_{e\in\E} l_e\,\nu_e.
    \end{align*}
    Inserting this estimate into the entropy inequality \eqref{eq:entropy_entropy_diss_proof} above yields 
    \begin{align*}
		\|\partial_x \rho^k\|_{L^2(\E)}^2 \leq C +   \frac1\tau \prt*{E(\rho^{k-1}) - E(\rho^k)}.
    \end{align*}
    From Eq.~\eqref{eq:entropy_entropy_diss_proof} we get
    \begin{align*}
        \tau \prt*{\|\partial_x w^k\|_{L^2(\E)}^2 + \|w^k\|_{L(\E)}^2}  \leq C +   \frac1\tau \prt*{E(\rho^{k-1}) - E(\rho^k)}.
    \end{align*}
    For the piecewise constant in time interpolation $w_\tau$ of $w^k$ this yields, summing from $k=1, \ldots, N_T$, the estimate
    \begin{align*}
        \|\partial_x w_\tau\|_{L^2(0,T;L^2(\E))}^2 + \|w_\tau\|_{L^2(0,T;L^2(\E))}^2 \leq C\,N_T + \sum_{k=1}^{N_T} \frac1\tau\prt*{E(\rho^{k-1}) - E(\rho^k)}.
    \end{align*}
    Since the sum on the right-hand side is telescopic, this simplifies further to
    \begin{align*}
       \tau\left(\|\partial_x w_\tau\|_{L^2(0,T;L^2(\E))}^2 + \|w_\tau\|_{L^2(0,T;L^2(\E))}^2\right) \leq C\,T + (E(\rho^0) - E(\rho^{N_T})),
    \end{align*}
    using $T=N_T\,\tau$.
    Since $0\leq \rho_e^k \leq 1$ for all $e \in \E$ and $\max_{e\in \E}l_e < \infty$, we obtain the following uniform estimate
    \begin{align*}
        \tau^{1/2} \|w_\tau\|_{L^2(0,T; H^1(\E))} \leq C,
    \end{align*}
    independent of $\tau>0$.
\end{proof}
\begin{lemma}[Time regularity for $\rho_\tau$]
	\label{lem:time-reg}
    Let $(\rho^k)_{k=0}^\infty \subset \A \cap H^1(\E)$ be the solution to the implicit Euler approximation (Eq.~\eqref{eq:weak_ent}) and let $\rho_\tau$ be the piecewise constant interpolation associated with $(\rho^k)_{k=0}^\infty$. 
    Then, there holds
    $$
    	\norm{d_\tau \rho_\tau}_{L^2(0,T; (H^1(\E))')} \leq C,
    $$    
    where $C>0$ is independent of $\tau>0$ and $d_\tau$ denotes the finite difference quotient
    $$
        [d_\tau \rho_\tau]|_{(t_{k-1},t_k]} = \frac{\rho^k - \rho^{k-1}}{\tau},\quad k=1,\ldots,N_T.
    $$
\end{lemma}
\begin{proof}
This result follows from the regularity estimates of Proposition~\ref{prop:discrete-entropy}. They allow to estimate the terms on the right hand side \eqref{eq:weak_ent} in terms of a constant multiplied by $\|\varphi\|_{H^1(\E)}$. Thus, taking the supremum over all $\varphi$ yields the desired $(H^1(\E))'$-estimate.
\end{proof}

Having established the a priori estimates, let us now show the existence of convergent subsequences whose limits we identify as weak solutions to \eqref{eq:weak}.

The bounds provided by Proposition~\ref{prop:discrete-entropy} in conjunction with the Banach-Alaoglu theorem (see \cite{Brezis2010}) yield the existence of subsequences and a function $\partial_x \rho\in L^2(0,T;L^2(\E))$, such that
\begin{itemize}
   \item $\partial_x \rho_\tau \rightharpoonup \partial_x \rho$ in $L^2(0, T; L^2(\E))$
\end{itemize}
where we did not relabel the subsequences. Moreover, again by the uniform bounds of Proposition~\ref{prop:discrete-entropy}, we may invoke \cite{simon1986compact}[Theorem 6] such that
\begin{itemize}
    \item $\rho_\tau\rightarrow \rho$ in $L^2(0, T; L^2(\E))$,
\end{itemize}
again, up to subsequences. Finally, from Lemma~\ref{lem:time-reg}, we have
\begin{itemize}
    \item $d_\tau \rho_\tau\rightharpoonup \partial_t \rho$ in $L^2(0, T; (H^1(\E))')$,
\end{itemize}
up to a subsequence. The identification of the limits follows from standard arguments for weak convergence, see, e.g. \cite{Crossley2025}[Section 2.3]. 

Having collected sufficient compactness and the corresponding convergent subsequences and limits, we can now prove the main result.
\begin{proof}[Proof of Theorem~\ref{thm:existence}]
Let us revisit Eq.~\eqref{eq:weak_ent}, \textit{i.e.},
\begin{multline}\label{eq:weak_ent_fvm}
\sum_{e \in \E} \int_0^T\int_e  \left(d_\tau \rho_{\tau,e}\,\varphi_{\tau,e} +
\rho_{\tau,e}\,(1-\rho_{\tau,e})\partial_x w_{\tau,e}\,\partial_x \varphi_{\tau,e} \right)dx + \tau \sum_{e \in \E} \int_0^T\int_e (\partial_x w_{\tau,e}\,\partial_x \varphi_{\tau,e} + w_{\tau,e}\,\varphi_{\tau,e})\;dx \\
+\sum_{v \in \V_D}\int_0^T (-\uinflow_{\tau,v}(t) (1-\rho^{k+1}(v)) + \uoutflow_{\tau,v}(t) \rho_{\tau}(v))\varphi_{\tau}(v) \;dt= 0,
\end{multline}
First let us note that the term premultiplied by $\tau$ vanishes due estimate~\eqref{eq:sqrt_tau_w_est}.
Next, using the convergences above, we can pass to the limit in the other terms of the equation to get
\begin{multline*}
\sum_{e \in \E} \int_e  \left(\partial_t \rho_e(t)\,\varphi_e +
(\varepsilon\,\partial_x \rho_e(t) - \nu_e \, 
f(\rho_e(t)))\,\partial_x \varphi_e \right)dx\\
+\sum_{v \in \V_D} (-\uinflow_{v}(t) (1-\rho(t,v)) + \uoutflow_v(t) \rho(t,v))\varphi(v) = 0,
\end{multline*}
for any $\varphi \in C_c^\infty((0,T)\times\E)$ which is dense in $L^2(0,T;H^1(\E))$. Thus, the limit $\rho$ is a weak solution to Eq.~\eqref{eq:weak}.

The a priori estimates 
follow from passing to the limit in the bounds of Proposition~\ref{prop:discrete-entropy} and Lemma~\ref{lem:time-reg}, using the weak lower semicontinuity of the norms. Finally, the compactness is sufficient to conclude that the weak solution satisfy the initial data.
\end{proof}

	\section{Numerical solvers}
	
	

 \label{sec:FVM}

 
 Partial differential equations (PDEs) are an essential tool in science and engineering, as they are typically used to model the complex physical phenomena. These equations are typically dependent on crucial system parameters that are mostly not known precisely and the formulation of the problem is written in an infinite-dimensional function space setting. As a result numerical discretizations of the equations are performed based. We here focus on the case when a finite volume method \cite{leveque2002finite} is used which was previously introduced in \cite{Blechschmidt2022}. These are popular discretization schemes as they usually work in a structure preserving manner. 

\subsection{Finite volume scheme}
 To derive a finite volume scheme we briefly recall our setup and start from considering differential operators defined on each edge, and we focus on non-linear drift-diffusion equations
\begin{equation}
\label{eq:strong_pde_app}
\partial_t\rho_e  = \partial_x (\eps\,\partial_x\rho_e - \nu_e f(\rho_e)), \quad e \in \E,
\end{equation}
where $\rho_e : e \times (0,T) \to \R_+$ describes, on each edge, the concentration of some quantity while $\nu_e>0$ an edge-dependent velocity, and $\eps > 0$ is a (typically small) diffusion constant. Furthermore, $f: \R_+ \to \R_+$ satisfies $f(0) = f(1) = 0$. This property ensures that solutions satisfy $0 \le \rho_e \le 1$ a.e. on each edge, see Theorem \ref{thm:existence}. By identifying each edge with an interval $[0,\ell_e]$, we define the flux as
\begin{align} \label{eq:flux_app}
J_e(x) := - \eps\,\partial_x \rho_e (x) + \nu_e f(\rho_e(x)).
\end{align}
A typical choice for $f$ used in the following is $f(\rho_e) = \rho_e(1-\rho_e)$.

The edge set incident to a vertex $v\in \V$ is
denoted by $\E_v$ and we distinguish among $\E_v^{\text{in}}\:=\{e\in \E\colon
e=(\widetilde v,v)\ \text{for some}\ \widetilde v\in \V\}$ and
$\E_v^{\text{out}} = \E_v\setminus \E_v^{\text{in}}$.
The control volumes are defined as follows. To each edge $e\in
\E$ we associate an equidistant grid of the parameter domain
\begin{equation*}
0 = x^e_{-1/2} < x^e_{1/2} <\ldots < x^e_{n_e+1/2} = L_e
\end{equation*}
with $h_e=x^e_{k+\frac12} - x^e_{k-\frac12}$, and introduce
the intervals $I_k^e = (x_{k-1/2}, x_{k+1/2})$ for all
$k=0,\ldots,n_e$.
We introduce the following control volumes for our finite volume method,
\begin{itemize}
	\item the interior edge intervals $I_1^e,\ldots,I_{n_e-1}^e$ for
	each $e\in \E$,
	\item the vertex patches $I^v = \big(\cup_{e\in \E_v^{\text{in}}} I_{n_e}^e\big)
	\cup \big(\cup_{e\in \E_v^{\text{out}}} I_0^e\big)$ for each $v\in \V$.
\end{itemize}
A semi-discrete approximation of the problem
\eqref{eq:strong_pde}--\eqref{eq:Dirichlet_conditions}
can be expressed by the volume averages
\begin{align*}
\rho_k^{e}(t) &= |I_k^e|^{-1}\int_{I_k^e} \rho_e(t,x)\; d x,\\
\rho^{v}(t) &= |I^v|^{-1} \Big(
\sum_{e\in \E_v^{\text{out}}} \int_{I^e_0}\rho_e(t,x)\; d x +
\sum_{e\in \E_v^{\text{in}}} \int_{I^e_{n_e}}\rho_e(t,x)\; d x
\Big),
\end{align*}
for all $e\in \E$, $k=1,\ldots,n_e-1$, resp.\ $v\in \V$.
With the definition of the vertex patches we strongly enforce the continuity
in the graph nodes. Integrating \eqref{eq:strong_pde_app} over some interval $I^e_k$,
$k=0,\ldots,n_e$, $e\in \E$, gives
\begin{align}
\label{eq:finite_volume_integral}
\int_{I_k^e}\,\partial_t \rho_e(t,x) \;d x
&=
\int_{I_k^e} \partial_x (\varepsilon\,\partial_x\rho_e(t,x) -\nu_e
f(\rho_e(t,x))\,d_e(t))\;d x \nonumber\\
=h_e\,\partial_t \rho_k^e &=
\left(
\varepsilon\,\partial_x\rho_e(t,x) -\nu_e
f(\rho_e(t,x))\,d_e(t)
\right)\Big\vert_{x^e_{k-1/2}}^{x^e_{k+1/2}}.
\end{align}
The diffusive fluxes are approximated by central differences
\begin{equation*}
\partial_x\rho(t,x^e_{k+1/2}) \approx \frac1{h_e}(\rho_{k+1}^e(t)-\rho_k^e(t))
\end{equation*}
and for the convective fluxes we use, for stability reasons, the
Lax-Friedrichs numerical flux
\begin{align}\label{eq:lax_friedrichs_flux}
f(\rho_e(t,x_{k+1/2}))\,d_e(t) &\approx F^e_{k+1/2}(t)\nonumber\\
&:= \frac{\nu_e}{2} (f(\rho_k^e(t)) + f(\rho_{k+1}^e(t)))\,d_e(t)
- \frac{\alpha}2 (\rho_{k+1}^e(t) - \rho_k^e(t)),
\end{align}
where we use the convention $\rho_0^e = \rho^v$ for $v\in V$
satisfying $e\in \E_v^{\text{out}}$ and $\rho_{n_e}^e =
\rho^{\widetilde v}$ with $\widetilde v\in V$ satisfying $e\in
\E_{\widetilde v}^{\text{in}}$.
The parameter $\alpha>0$ is some stabilization parameter, chosen
sufficiently large. At inflow and outflow vertices $v\in
\V_\D$ we insert the boundary condition
\eqref{eq:Dirichlet_conditions} into \eqref{eq:finite_volume_integral}
and obtain
\begin{equation*}
\sum_{e\in \E_v} \left(\varepsilon\,\partial_x\rho_e(t,v) -
\nu_e f(\rho_e(t,v))\,d_e(t)\right)
\approx -\uinflow_{v}(t)\,(1-\rho^v) + \uoutflow_{v}(t)\,\rho^v.
\end{equation*}
Combining the previous investigations gives the following set of
equations for each control volume $I_k^e$, $k=1,\ldots,n_e-1$, $e\in \E$, and $I^v$, $v\in \V$,
respectively.
\begin{subequations}
	\label{eq:semi_discrete_fvm}
	\begin{align}
	\intertext{For each $e\in \E$ and $k=1,\ldots,n_e-1$:}
	h_e\,\partial_t \rho_k^e(t) + \varepsilon\,\frac{-\rho_{k-1}^e(t) +
		2\rho_k^e(t) - \rho_{k+1}^e(t)}{h_e} - F_{k-\frac12}^e(t) + F_{k+\frac12}^e(t) &= 0.
	\\
	\intertext{For each $v\in \V_\K$:}
	\sum_{e\in \E_v} h_e\,\partial_t\rho^v(t)
	+ \sum_{e\in \E_v^{\text{in}}}
	\left(\varepsilon\,\frac{\rho^v(t)-\rho_{n_e-1}^e(t)}{h_e} -
	F^e_{n_e-\frac12}(t)\right) &\nonumber\\
	- \sum_{e\in \E_v^{\text{out}}}
	\left(\varepsilon\,\frac{\rho_1^e(t)- \rho^v(t)}{h_e} - F^e_{\frac12}(t)\right)
	&= 0.\label{eq:fvm_kirchhoff_vertices}\\
	\intertext{For each influx node $v\in \V_\D^{\text{in}}$:}
	\sum_{e\in \E_v} h_e\,\partial_t\rho^v(t)
	- \sum_{e\in \E_v^{\text{out}}}
	\left(\varepsilon\,\frac{\rho_1^e(t)-\rho^v(t)}{h_e} -
	F^e_{\frac12}(t)\right)
	-\uinflow_{v}\,(1-\rho^v(t))
	&= 0.
	\intertext{For each outflux node $v\in \V_\D^{\text{out}}$:}
	\sum_{e\in \E_v} h_e\,\partial_t\rho^v(t)
	+ \sum_{e\in \E_v^{\text{in}}}
	\left(\varepsilon\,\frac{\rho^v(t)-\rho_{n_e-1}^e(t)}{h_e} -
	F^e_{n_e-\frac12}(t)\right)
	+ \uoutflow_{v}\,\rho^v(t)
	&= 0.
	\end{align}
\end{subequations}
In \eqref{eq:fvm_kirchhoff_vertices} accumulated contributions
evaluated in $v$ vanish due to the Kirchhoff-Neumann vertex conditions
\eqref{eq:Kirchhoff_Neumann_condition}.

To solve the system of ordinary differential equations
\eqref{eq:semi_discrete_fvm} for the unknowns
$\rho^e_k$ and $\rho^v$, respectively, we introduce the following
time-discretization.
For some equidistant time grid $0=t_0<t_1<\ldots<t_{n_t}=T$ with grid
size $\tau = t_n - t_{n-1}$, $n=1,\ldots,n_t$, we define the following
grid functions by
\begin{equation*}
\rho^{v,n} = \rho^v(t_n),\quad \rho^{e,n}_k = \rho^e_k(t_n),
\quad F_{k+1/2}^{e,n} = F_{k+1/2}^e(t_n).
\end{equation*}
We restrict the equations \eqref{eq:semi_discrete_fvm} to the grid points
and replace the time derivative by a difference quotient, evaluate
the diffusion terms in $t_{n+1}$ and the convective terms in $t_n$.
This yields for each $n=1,\ldots,n_t$
the following system of equations:
\begin{subequations}
	\label{eq:fully_discrete_fvm}
	\begin{align}
	\intertext{For each $e\in \E$ and $k=1,\ldots,n_e-1$:}
	h_e\,\rho_k^{e,n} + \varepsilon\,\tau\,\frac{-\rho_{k-1}^{e,n} +
		2\rho_k^{e,n} - \rho_{k+1}^{e,n}}{h_e}
	&= h_e\,\rho_k^{e,n-1}
	+ \tau\left(F_{k-\frac12}^{e,n-1} -
	F_{k+\frac12}^{e,n-1}\right).
	\label{fully_discrete_fvm_edge_equation}
	\\
	\intertext{For each $v\in \V_\K$:}
	\abs{I_v}\,\rho^{v,n}
	+ \tau\,\varepsilon\,\sum_{e\in \E_v^{\text{in}}}
	\,\frac{\rho^{v,n}-\rho_{n_e-1}^{e,n}}{h_e}
	&- \tau\,\varepsilon\,\sum_{e\in \E_v^{\text{out}}}
	\frac{\rho_1^{e,n}- \rho^{v,n}}{h_e} \nonumber\\
	= \abs{I_v}\,\rho^{v,n-1} + \tau\,\sum_{e\in \E_v^{\text{out}}}F^{e,n-1}_{\frac12}
	&- \tau\,\sum_{e\in \E_v^{\text{in}}}F^{e,n-1}_{n_e-\frac12}.\\
	\intertext{For each influx node $v\in \V_\D^{\text{in}}$:}
	\abs{I_v}\,\rho^{v,n}
	- \tau\,\varepsilon\sum_{e\in \E_v^{\text{out}}}
	\frac{\rho_1^{e,n}-\rho^{v,n}}{h_e}
	&= \abs{I_v}\,\rho^{v,n-1} +
	\tau\,F^{e,n-1}_{\frac12}
	+ \tau\,\uinflow_{v}\,(1-\rho^{v,n-1}). \\
	\intertext{For each outflux node $v\in \V_\D^{\text{out}}$:}
	\abs{I_v}\,\rho^{v,n}
	+ \tau\,\varepsilon\sum_{e\in \E_v^{\text{in}}}
	\frac{\rho^{v,n}-\rho_{n_e-1}^{e,n}}{h_e}
	&= \abs{I_v}\,\rho^{v,n-1} + \tau\,\sum_{e\in
		\E_v^{\text{in}}} F^{e,n-1}_{n_e-\frac12}
	-\tau\,\uoutflow_{v}\,\rho^{v,n-1}.
	\end{align}
\end{subequations}
The initial data are established by
\begin{equation*}
\rho_k^{e,0}=\pi_{I_k^e}(\rho_0),\qquad \rho^{v,0} = \pi_{I^v}(\rho_0),
\end{equation*}
where $\pi_M$ denotes the $L^2$-projection onto the constant functions on
a subset $M\subset \Gamma$.
Note that this set of equations is linear in the unknowns in the new time
point $\rho_k^{e,n}$, $k=1,\ldots,n_e-1$, $e\in \E$ and $\rho^{v,n}$,
$v\in \V$.
The fully-discrete approximation $\widetilde\rho\colon [0,T]\times \Gamma\to
\R$ then reads
\begin{equation*}
\widetilde\rho(t,x) = \widehat\rho^n(x)\quad \text{for}\ t\in
[t_n,t_{n+1}),
\end{equation*}
with
\begin{equation*}
\widehat\rho^n(x) = \rho^{v,n}\ \text{for}\ x\in I^v,\qquad \widehat\rho^n(x)
= \rho_k^{e,n}\ \text{for}\ x\in I_k^e.
\end{equation*}

It is well-known that finite-volume schemes like \eqref{eq:fully_discrete_fvm}
guarantee a couple of very important properties. On the one hand, there is a
well established convergence theory, see e.\,g.\
\cite{MortonStynesSuli1997,LazarovMishevVassilevski1996,ThijeBoonkkampAnthonissen2010}.
On the other hand, our scheme is
mass-conserving and bound-preserving which we show in the following theorem.
Thus, the finite volume approach is suitable of generating reference solutions used to train the \deeponet\ models proposed here. 
\begin{theorem}
	The solution of \eqref{eq:fully_discrete_fvm}, $\widetilde\rho$, satisfies the following
	properties:
	\begin{enumerate}[label=\roman*)]
		\item The scheme is mass conserving, i.e., if $\uinflow_{v}\equiv \uoutflow_{v}\equiv 0$ for all $v\in \V_{\D}$,
		then there holds
		\begin{equation*}
		\int_\Gamma\widehat\rho^n\; d x = \int_\Gamma\widehat\rho^0\; d x\qquad\forall n=1,\ldots,n_t.
		\end{equation*}
		\item Assume that $f(x) = x(1-x)$ and in
		\eqref{eq:lax_friedrichs_flux} choose $\alpha=1$. Then, the scheme is
		bound-preserving, i.e., there holds
		\begin{equation*}
		\widetilde\rho(t,x)\in [0,1]\qquad \forall t\in [0,T],
		x\in \Gamma,
		\end{equation*}
		provided that $\tau\le \min_{e\in \E}h_e$.
	\end{enumerate}
\end{theorem}
\begin{proof}
	\begin{enumerate}[label=\roman*)]
		\item This directly follows after summing up all the equations in
		\eqref{eq:fully_discrete_fvm}. Note that the diffusive and convective
		fluxes cancel out.
		\item The system \eqref{eq:fully_discrete_fvm} can be written as a
		system of linear equations of the form
		\begin{equation}
		\label{eq:fvm_equation_system}
		(M+\tau\,\varepsilon A)\vec \rho^n = M\vec \rho^{n-1} + \vec F(\vec\rho^{n-1}),
		\end{equation}
		where $M$ is the mass matrix and $A$ contains the coefficients of the
		diffusion terms on the left-hand side of
		\eqref{eq:fully_discrete_fvm}.
		The vector $\vec \rho^n$ contains the unknowns
		$\rho^{v,n}$ and $\rho_k^{e,n}$. In the usual ordering of
		unknowns and equations the matrix $M+\tau\,\varepsilon\,A$
		is strictly diagonal dominant and is thus an M-matrix.
		The inverse possesses non-negative entries only.
		The right-hand side of \eqref{eq:fvm_equation_system} is
		also non-negative under the assumption $\vec\rho^{n-1}\in
		[0,1]$. We demonstrate this for the equation
		\eqref{fully_discrete_fvm_edge_equation}. Insertion
		of \eqref{eq:lax_friedrichs_flux} and reordering the terms
		yields
		\begin{align*}
		&h_e\,\rho_k^{e,n-1} + \tau\left(F_{k-1/2}^{e,n-1} -
		F_{k+1/2}^{e,n-1}\right) = (h_e-\alpha\,\tau)\,\rho_k^{e,n-1}\\
		&\quad 
		+\frac\tau2\left((1-\rho_{k-1}^{e,n-1}) + \alpha\right)
		\rho_{k-1}^{e,n-1}
		+
		\frac\tau2\left(-(1-\rho_{k+1}^{e,n-1}) + \alpha\right)
		\rho_{k+1}^{e,n-1} \ge 0.
		\end{align*}
		The non-negativity follows from $\rho_{k}^{e,n-1}\in
		[0,1]$ for $k=0,\ldots,n_e$ and $\alpha=1$ as well as
		$\tau\le \min_{e\in \E}h_e$.
		This, together with the M-matrix property of
		$M+\tau\,\varepsilon\,A$, implies $\vec\rho^n\ge 0$.
		
		Due to
		$f(x)=x(1-x)$ we may rewrite
		\eqref{eq:fvm_equation_system} in the form
		\begin{equation*}
		(M+\tau\,\varepsilon A)(\vec1-\vec \rho^n)
		= M(\vec1-\vec\rho^{n-1}) + \vec G(\vec 1-\vec\rho^{n-1}),
		\end{equation*}
		with some vector-valued function $\vec G$. With similar
		arguments like before we conclude that the right-hand side
		is non-negative and	    
		thus, $1-\vec\rho^n \ge 0$, which proves the upper bound.
		By induction the result follows for all $n=1,\ldots,n_t$.
	\end{enumerate}
\end{proof}

\subsection{\deeponet\ further details }

The mathematical foundation of \deeponet\ is rooted in the concept of approximating operators, which are mappings between infinite-dimensional function spaces. Let \(\mathcal{G}\) be such an operator, which maps the input function \(u(x)\) to an output function \(G(u)(y)\), where \(x \in \mathbb{R}^d\) and \(y \in \mathbb{R}^m\) are the input and output coordinates, respectively. The goal of a DeepONet is to approximate \(\mathcal{G}\) using a neural network architecture that can handle functional inputs and outputs.

In more detail, the \deeponet\ architecture consists of two main ingredients: namely, the \textit{branch net} and the \textit{trunk net}. The branch net takes as input the discretized values of the input function \(u(x)\) at a set of predefined training points \(\{x_1, x_2, \dots, x_n\}\), and gives as an output a set of coefficients \(\{b_1, b_2, \dots, b_p\}\). On the other hand, the trunk net takes as input the output coordinate \(y\) and outputs a set of basis functions \(\{t_1(y), t_2(y), \dots, t_p(y)\}\). The final output of the \deeponet\ is then given by the inner product of the branch and trunk outputs written as
\[
G(u)(y) = \sum_{i=1}^p b_i(u) \cdot t_i(y),
\]
where in this equation the coefficients \(b_i(u)\) are obtained from the branch net, and the coefficients \(t_i(y)\) are produced by the trunk net. With this we are able to approximate the operator \(\mathcal{G}\) by learning the appropriate coefficients and basis functions from training data.
\subsubsection*{Branch Net}
The branch net architecture is used for encoding the input function \(u(x)\) into a finite-dimensional representation. Given the discrete values of \(u(x)\) at known points \(\{x_1, x_2, \dots, x_n\}\), the branch net processes these values through a neural network architecture to produce \(\{b_1, b_2, \dots, b_p\}\). In summary, the branch net can be represented as a function \(B: \mathbb{R}^n \rightarrow \mathbb{R}^p\) such that:
\[
\mathbf{b} = B(u(x_1), u(x_2), \dots, u(x_n)),
\]
where \(\mathbf{b} = [b_1, b_2, \dots, b_p]^T\) is the vector collecting all the coefficients.

\subsubsection*{Trunk Net}
The trunk net is then used for generating the basis functions that are used to construct the output function. For the output coordinate \(y\), the trunk net processes \(y\) using a deep learning architecture to produce the basis functions \(\{t_1(y), t_2(y), \dots, t_p(y)\}\). Again, we obtain the following representation \(T: \mathbb{R}^m \rightarrow \mathbb{R}^p\) such that:
\[
\mathbf{t}(y) = T(y),
\]

where \(\mathbf{t}(y) = [t_1(y), t_2(y), \dots, t_p(y)]^T\) is the vector of basis functions. 

The training of a \deeponet\ involves minimizing a loss function that measures the discrepancy between the predicted output and the true output. Given a dataset of input-output pairs \(\{(u_i, G(u_i))\}_{i=1}^N\), the loss function \(\mathcal{L}\) is defined as:

\[
\mathcal{L} = \frac{1}{N} \sum_{i=1}^N \int \left\| G(u_i)(y) - \sum_{j=1}^p b_j(u_i) \cdot t_j(y) \right\|^2 dy,
\]

where the integral is taken over the domain of the output function. In practice, the integral is approximated using numerical integration techniques, such as Monte Carlo sampling or quadrature methods. The parameters of the branch and trunk nets are then optimized using gradient-based methods, such as stochastic gradient descent (SGD) or its variants, to minimize the loss function.

\section{Comments}
The DeepONet approach shares a lot of similarity (even equivalence) with the FNO approach as pointed out in \cite{kovachki2023neural}. As a result one could apply our surrogate coupling technique with a different choice of the DeepONet architecture to obtain an FNO setup and vice versa. Our methodology of coupling surrogate models based on the graph topology can be used with different operator learning methods, e.g., (physics-informed) DeepONets or FNO. Also, one could use a Graph Neural Operator technique that can act in the same way as the DeepONet or FNO for one edge operator thus creating a surrogate operator, see\cite{li2020neural}. 

\subsection{Discussion of strong GP Prior}
While we use the Gaussian (RBF) kernel in several places our main goal to avoiding an inverse crime was to choose different parameters for the generation of training data (inflow/outflow/initial: length scale $\ell=0.5$ and $512$ Gaussian centers), to generate random data for simulation (inflow/outflow/initial: length scale  $\ell=0.4$ and $468$ Gaussian centers).
In the inverse setting, we employ a length scale of $\ell=0.2$ and only $10$ Gaussian centers to learn the flow couplings and unknown initial conditions.

\subsection{Further applications}
Beyond the toy example of traffic flow considered here, at least two more applications come to mind: The first is the transport of cargo inside of biological cells that is realized by molecular motors traveling along a network of one-dimensional filaments (i.e. the graph in our setting). Previous work by different authors have demonstrated that, starting from an accurate microscopic model, a mean field limit produces exactly the type of drift-diffusion equations that we consider, \cite{Bressloff2015,Bressloff2016}. Finally, when studying the transport in gas networks, (non-linear variants) of drift-diffusion equations also appear as approximation to the (otherwise hyperbolic) governing equations. They go under the name ISO3 model for gas transport, \cite{DomschkeHillerLangetal2021}.

\subsection{Computational complexity}
For the pure simulation task, the FVM solver is typically faster than our method. This is a caveat of most physics-informed neural network and operator network approaches. However, our methodology shines in the inverse problem setting, were dedicated approaches are needed for each problem type. To the best of our knowledge, unfortunately, such solvers don't exist for our specific setting and a direct comparison is infeasible.
However, we estimate the complexity involved in both approaches for the inverse problem as follows:
\begin{align*}
    \text{FVM:} \quad &\mathcal{O}(N_\text{GradientSteps}\cdot N_t \cdot (n_e \cdot N_\text{edges}+N_\text{vertices})^2)\\
    \text{PI DeepONet:} \quad  & \mathcal{O}(
    N_\text{GradientSteps} \cdot ((3 \cdot n_\beta + 1) \cdot N_\text{edges})
    )
\end{align*}
and thus a reduction from quadratic to linear complexity.

 \section{Loss plots}%
 \label{sec:loss_plots}
The loss function evolution over \num{20000} epochs is illustrated in \cref{fig:lossapp1} for the model with width \num{100} and for width \num{200} in \cref{fig:lossapp2}.
 
\begin{figure}
    \centering
\begin{subfigure}[b]{0.32\textwidth}
    \includegraphics[width=0.99\linewidth]{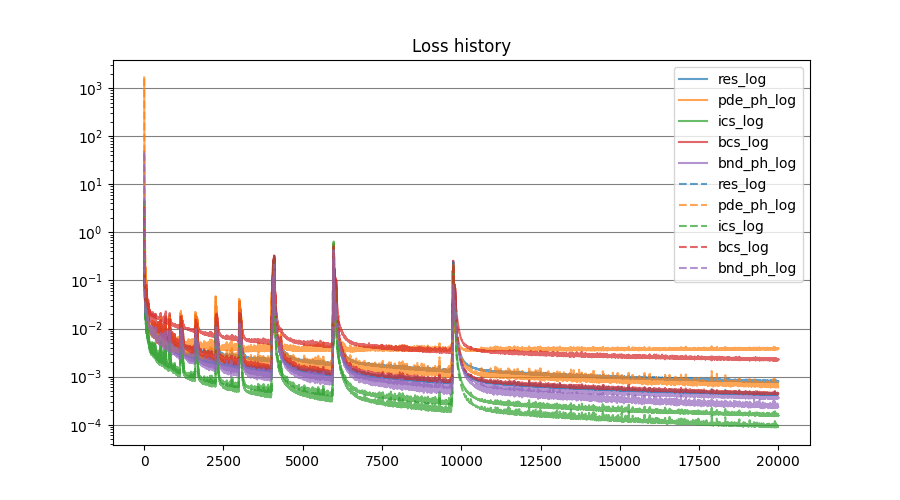}
    \includegraphics[width=0.99\linewidth]{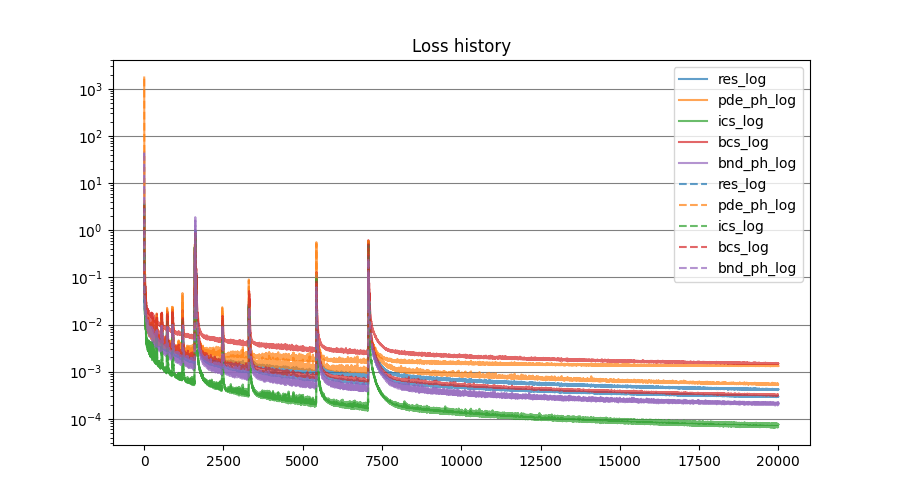}
    \includegraphics[width=0.99\linewidth]{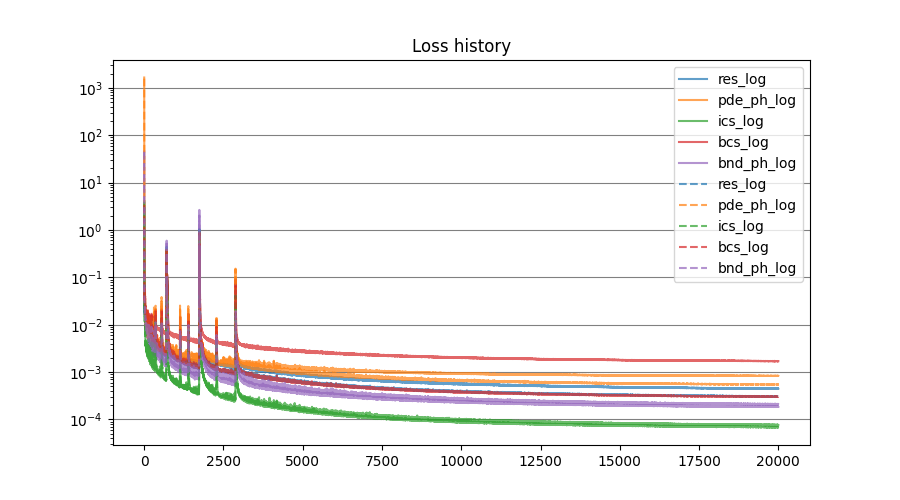}
    \caption{Inflow models}%
    \end{subfigure}\hfill%
    \begin{subfigure}[b]{0.32\textwidth}
    \includegraphics[width=0.99\linewidth]{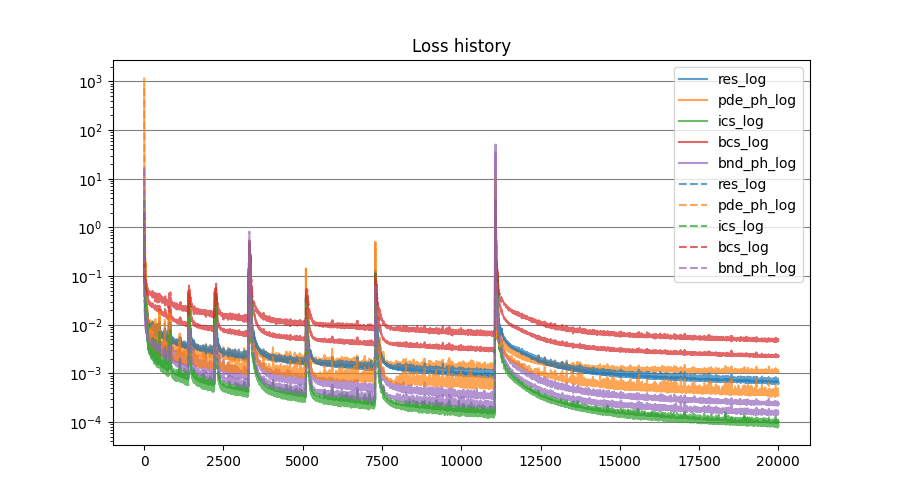}
    \includegraphics[width=0.99\linewidth]{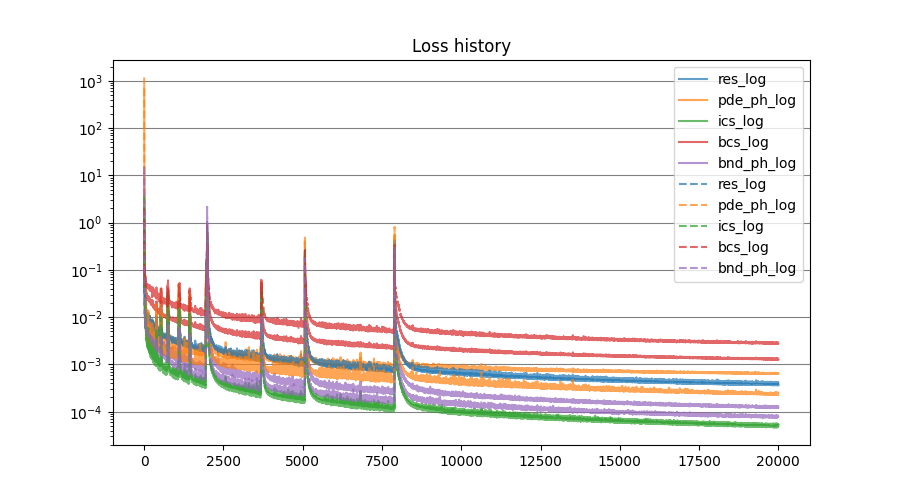}
    \includegraphics[width=0.99\linewidth]{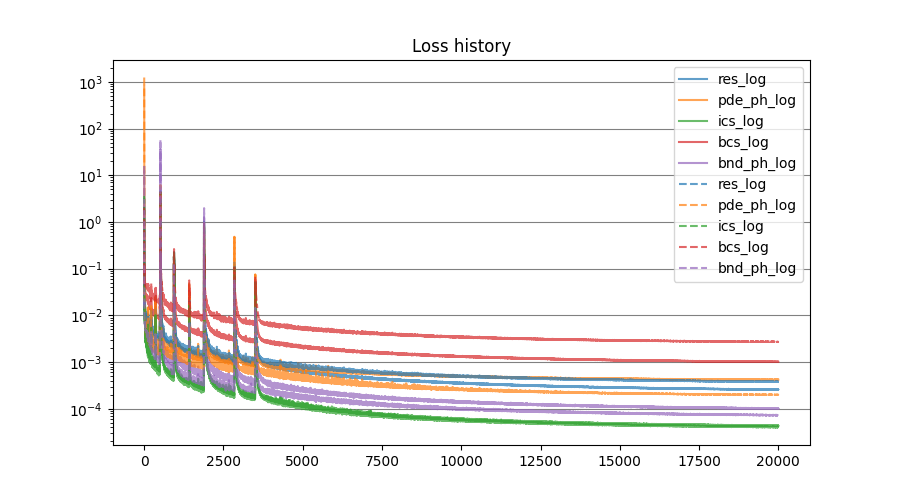}
    \caption{Inner models}
\end{subfigure}\hfill%
\begin{subfigure}[b]{0.32\textwidth}
    \includegraphics[width=0.99\linewidth]{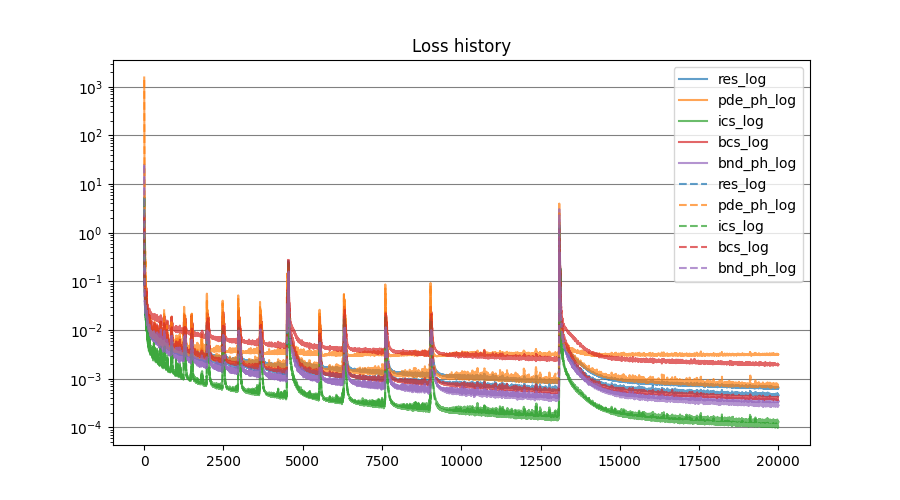}
    \includegraphics[width=0.99\linewidth]{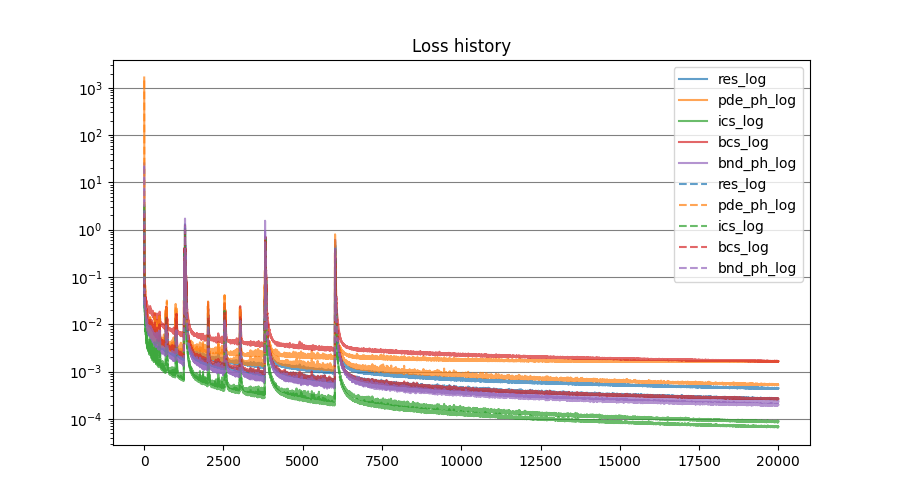}
    \includegraphics[width=0.99\linewidth]{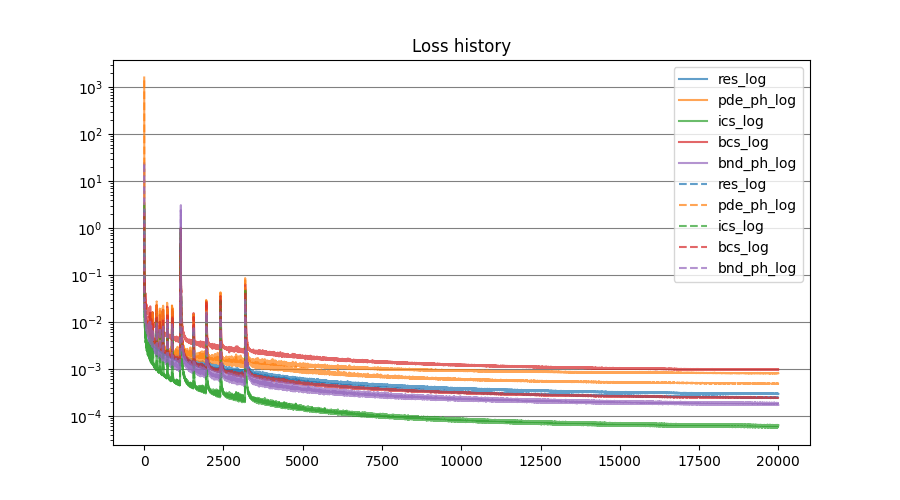}
    \caption{Outflow models}
\end{subfigure}%
    \caption{Terms of loss function in training of models with  \textbf{width 100} and 5K (up), 10K (middle), 20K (down) training data.
    Solid lines report training loss of various terms, dashed lines report validation loss.
    Recall that although only PDE physics loss (\texttt{pde\_ph\_log}), boundary physics loss (\texttt{bnd\_ph\_log}) and initial condition loss (\texttt{ics\_log}) are considered in objective, all terms decrease during training.
    The $x$-axis shows the learning epoch and the $y$-axis the value of the corresponding loss term.}
    \label{fig:lossapp1}
\end{figure}

\begin{figure}
    \centering
\begin{subfigure}[b]{0.32\textwidth}
    \includegraphics[width=0.99\linewidth]{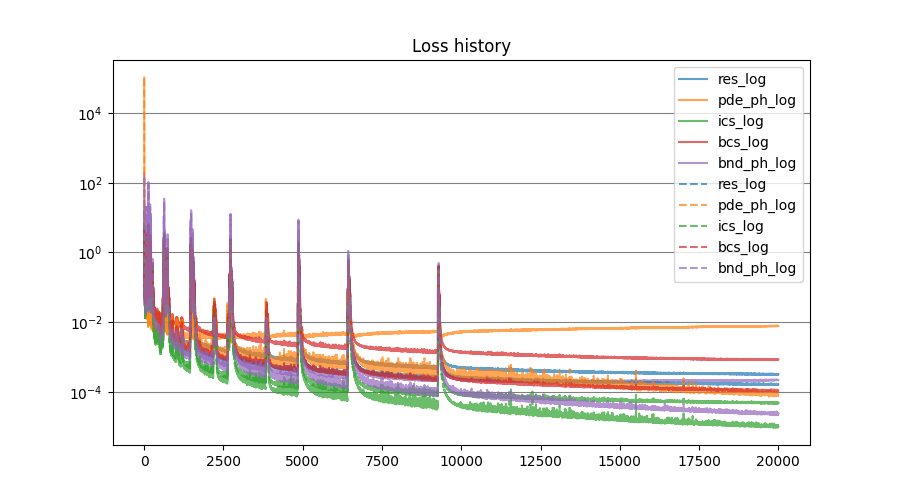}
    \includegraphics[width=0.99\linewidth]{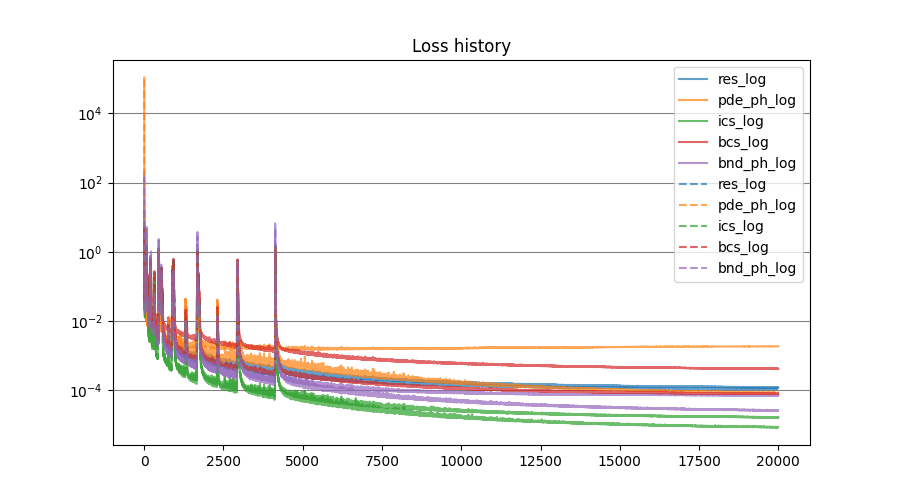}
    \includegraphics[width=0.99\linewidth]{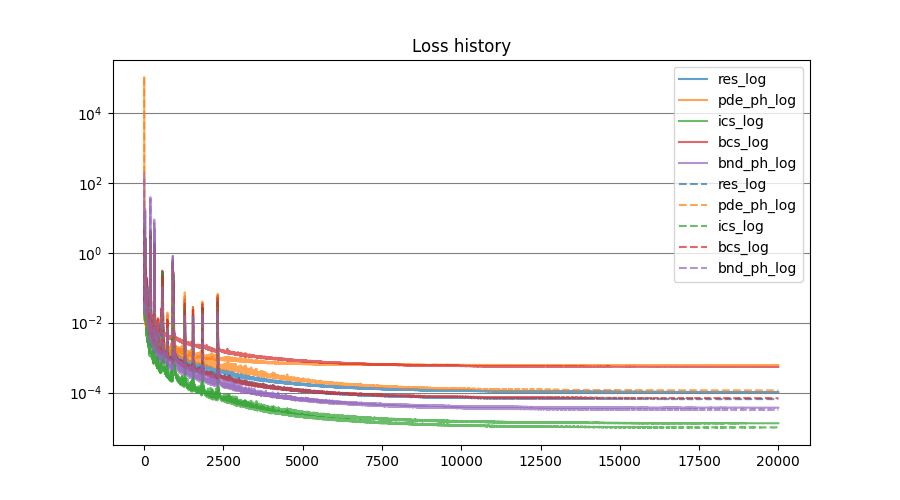}
    \caption{Inflow models}%
    \end{subfigure}\hfill%
    \begin{subfigure}[b]{0.32\textwidth}
    \includegraphics[width=0.99\linewidth]{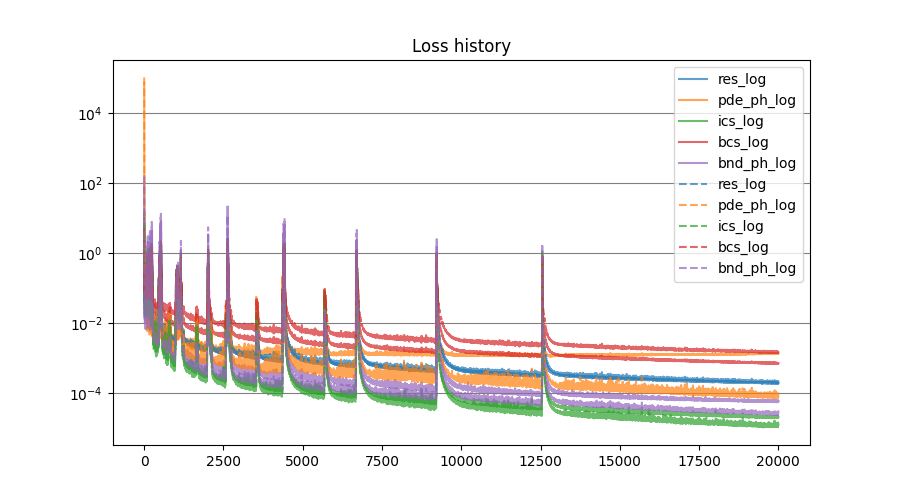}
    \includegraphics[width=0.99\linewidth]{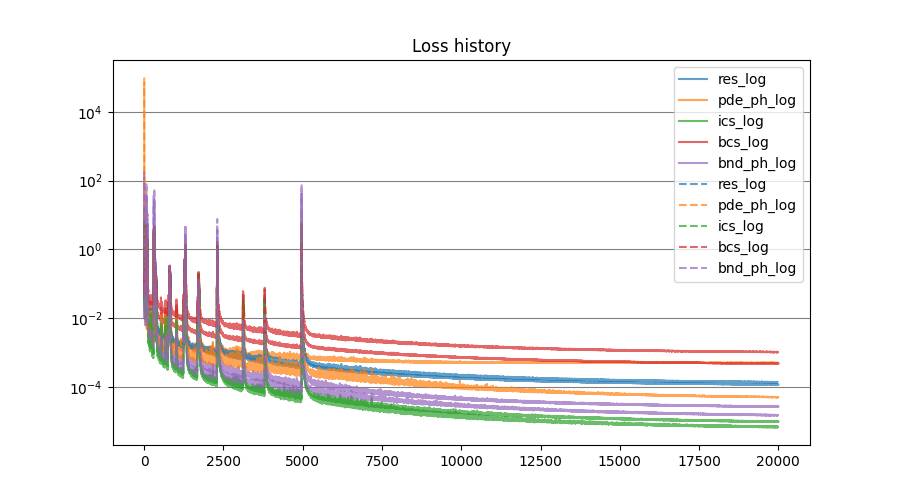}
    \includegraphics[width=0.99\linewidth]{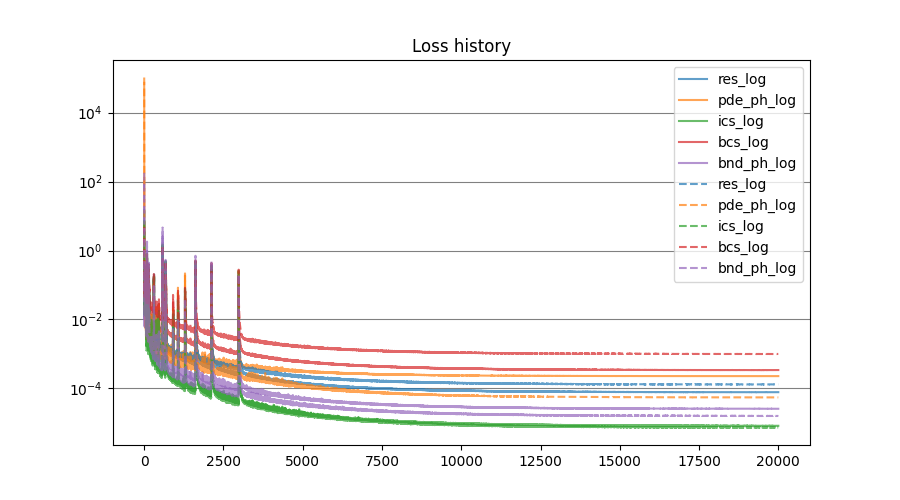}
    \caption{Inner models}
\end{subfigure}\hfill%
\begin{subfigure}[b]{0.32\textwidth}
    \includegraphics[width=0.99\linewidth]{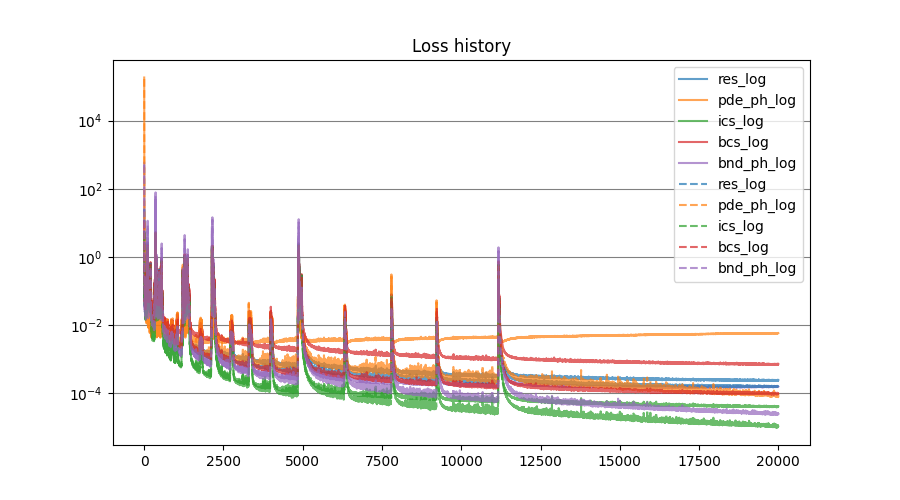}
    \includegraphics[width=0.99\linewidth]{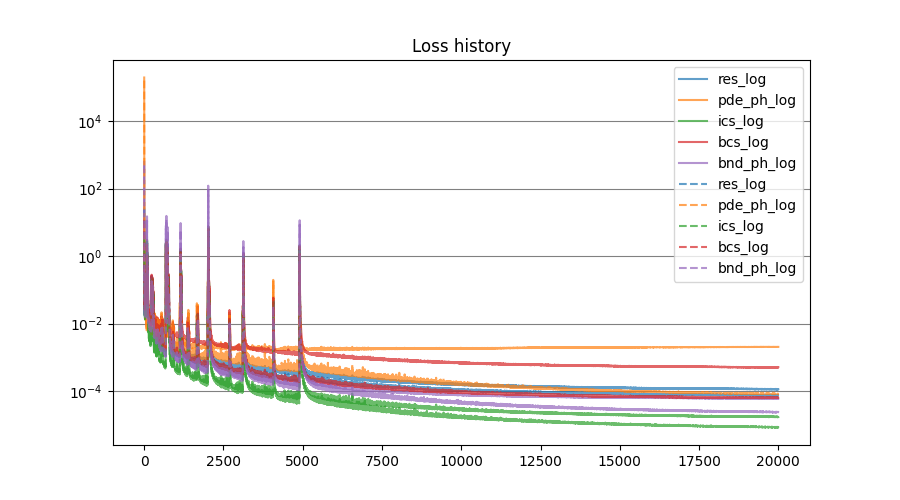}
    \includegraphics[width=0.99\linewidth]{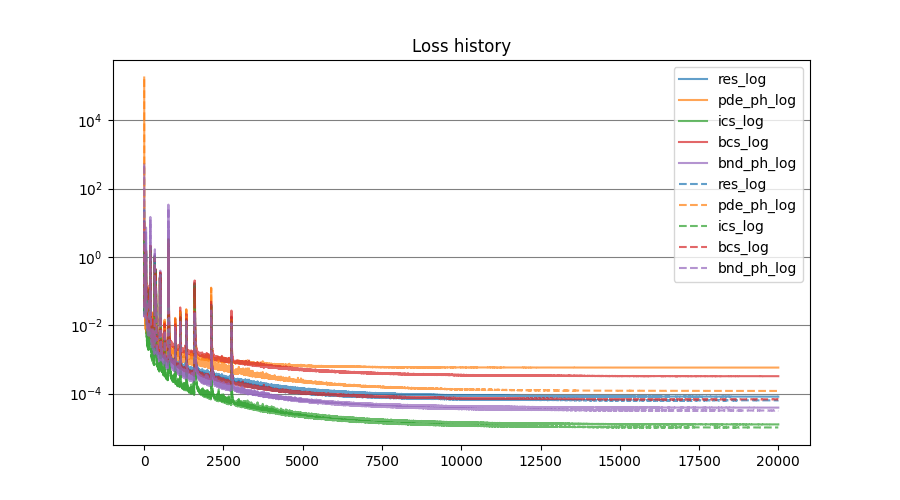}
    \caption{Outflow models}
\end{subfigure}%
    \caption{Terms of loss function in training of models with \textbf{width 200} and 5K (up), 10K (middle), 20K (down) training data.
    Solid lines report training loss of various terms, dashed lines report validation loss.
    Recall that although only PDE physics loss (\texttt{pde\_ph\_log}), boundary physics loss (\texttt{bnd\_ph\_log}) and initial condition loss (\texttt{ics\_log}) are considered in objective, all terms decrease during training.
    The $x$-axis shows the learning epoch and the $y$-axis the value of the corresponding loss term.}
    \label{fig:lossapp2}
\end{figure}

    \section{Results on large graphs}
 We test our method on larger graph networks with more edges using a directed network construction of varying sizes (102, 306, 1034 edges).
  This is easily done by providing the adjacency structure and the inflow and outflow nodes.
 These examples also contain multiple inflow and outflow nodes, see for example \cref{fig:graph1034}.

 \begin{figure}
     \centering
     \includegraphics[trim={3cm 0 3cm 3cm},clip,width=1\linewidth]{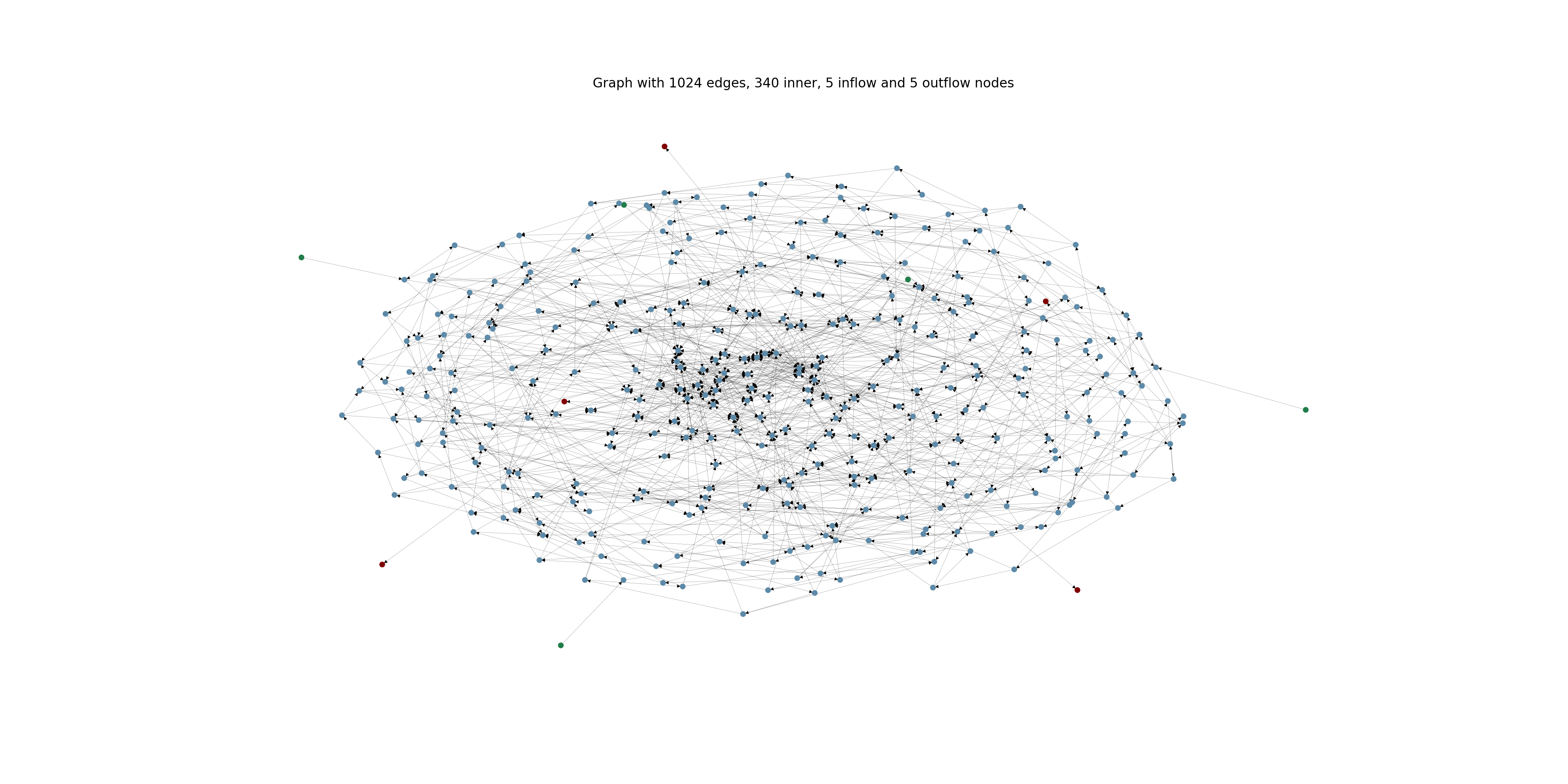}
     \caption{Graph with 1034 edges, 5 inflow nodes (green) and 5 outflow nodes (red).}
     \label{fig:graph1024}
 \end{figure}

 The local accuracy of our model is guaranteed by the accuracy of the surrogate model, in our case the physics informed DeepONet, which makes sure that the PDE is resolved up to the precision of the training of this model.
 The overall accuracy of the coupled surrogate models depends on the solution accuracy at the vertices, only, which we enforce for the least-squares solver.
 The scaling to large networks therefore depends on the robustness of the least squares solver.
 Here, we rely on the JAX implementation of an ADAM SGD method which should scale well with increasing network size utilizing the computational power of the underlying GPU.

 An exemplary simulation on the graph depicted in \cref{fig:graph1024} with noise 0.01 yields the following an absolute $L^2$ error of 1.67e-02 and a relative $L^2$ error of 4.84e-02.
 The absolute error of the solution at different time steps is shown in \cref{fig:simulation_large}.

 \begin{figure}
     \centering
     \includegraphics[trim={3cm 0 3cm 0},clip,width=0.9\linewidth]{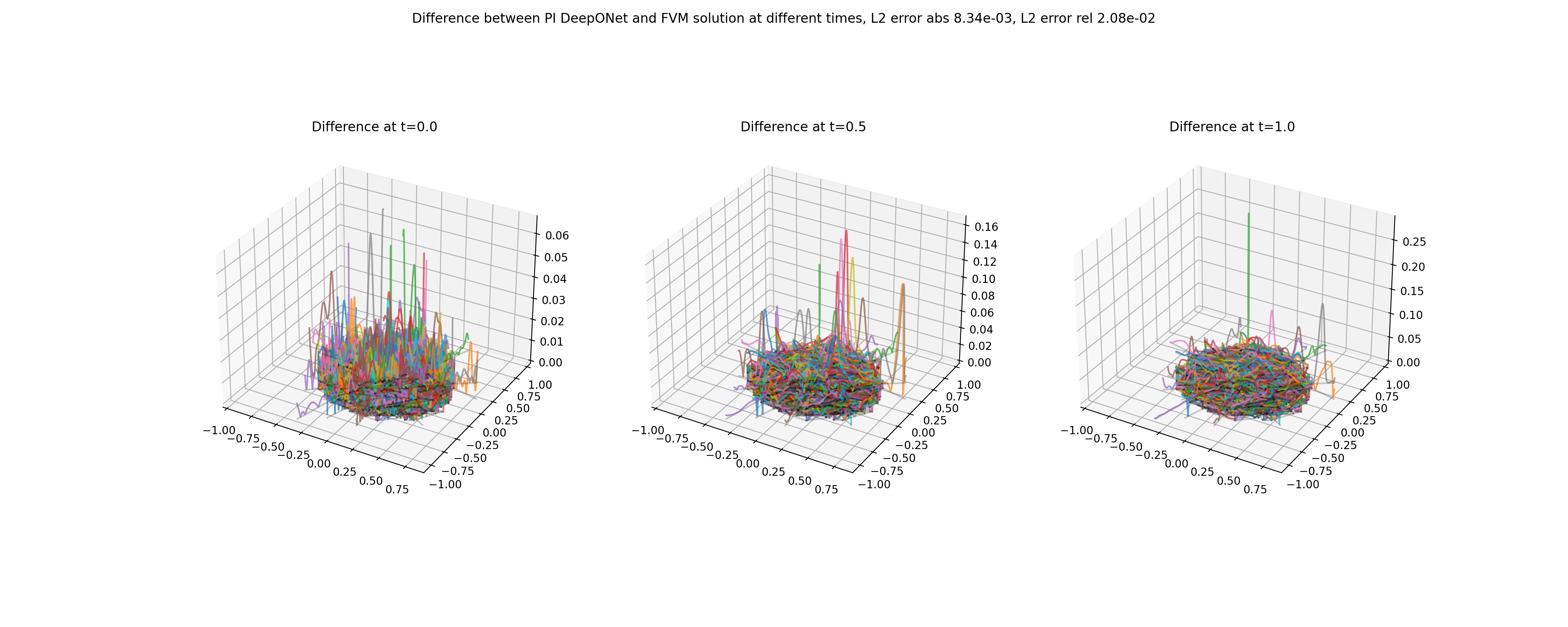}
     \caption{Absolute difference of PI DeepONet solution and baseline solution of example problem on graph with 1034 edges.}
     \label{fig:simulation_large}
 \end{figure}
 
An exemplary inverse problem on the graph depicted in \cref{fig:graph1024} with noise 0.01 yields the following errors: $L^2$ error of the solution on whole domain 1.57e-02 (abs) and 3.53e-02 (rel);
$L^2$ error of the initial condition 5.06e-02 (abs) and 1.11e-01 (rel);
$\ell^2$ error of edge velocity below 2.29e-02 (abs) and 2.34e-02 (rel).
 The absolute error of the solution at different time steps is shown in \cref{fig:inverse_large}.
  \begin{figure}
     \centering
     \includegraphics[trim={3cm 0 3cm 0},clip,width=0.9\linewidth]{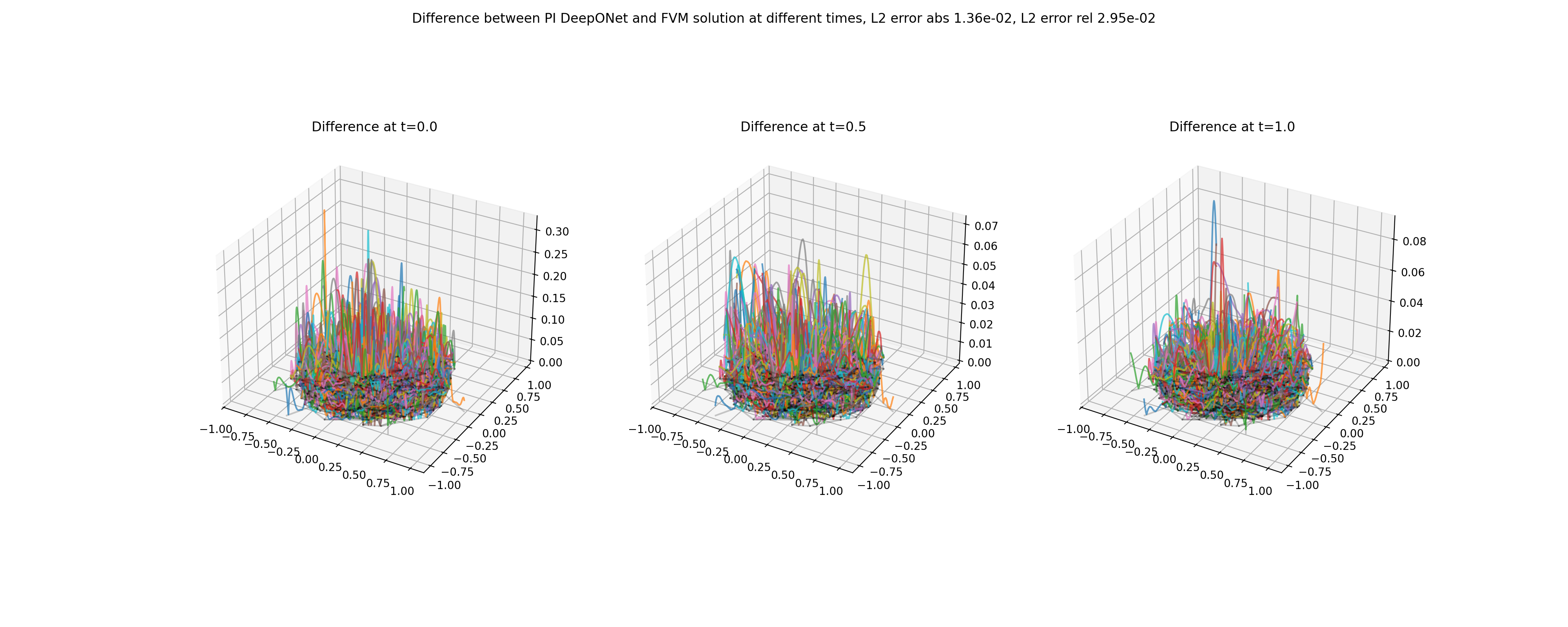}
     \caption{Absolute difference of PI DeepONet solution and baseline solution of example problem on graph with 1034 edges.}
     \label{fig:inverse_large}
 \end{figure}

\end{document}

%% file: architecture.pdf_tex
\begingroup%
  \makeatletter%
  \providecommand\color[2][]{%
    \errmessage{(Inkscape) Color is used for the text in Inkscape, but the package 'color.sty' is not loaded}%
    \renewcommand\color[2][]{}%
  }%
  \providecommand\transparent[1]{%
    \errmessage{(Inkscape) Transparency is used (non-zero) for the text in Inkscape, but the package 'transparent.sty' is not loaded}%
    \renewcommand\transparent[1]{}%
  }%
  \providecommand\rotatebox[2]{#2}%
  \newcommand*\fsize{\dimexpr\f@size pt\relax}%
  \newcommand*\lineheight[1]{\fontsize{\fsize}{#1\fsize}\selectfont}%
  \ifx\svgwidth\undefined%
    \setlength{\unitlength}{173.35067136bp}%
    \ifx\svgscale\undefined%
      \relax%
    \else%
      \setlength{\unitlength}{\unitlength * \real{\svgscale}}%
    \fi%
  \else%
    \setlength{\unitlength}{\svgwidth}%
  \fi%
  \global\let\svgwidth\undefined%
  \global\let\svgscale\undefined%
  \makeatother%
  \begin{picture}(1,1.81453597)%
    \lineheight{1}%
    \setlength\tabcolsep{0pt}%
    \put(0.13184661,1.70985227){\color[rgb]{0,0,0}\makebox(0,0)[lt]{\lineheight{1.25}\smash{\begin{tabular}[t]{l}$(t,x)$\end{tabular}}}}%
    \put(0.04256732,0.52492538){\color[rgb]{0,0,0}\makebox(0,0)[lt]{\lineheight{1.25}\smash{\begin{tabular}[t]{l}$\mathcal{H}(G_\theta^u(t,x))$\end{tabular}}}}%
    \put(0.22228038,0.26590298){\color[rgb]{0,0,0}\makebox(0,0)[lt]{\lineheight{1.25}\smash{\begin{tabular}[t]{l}$\mathcal{L}_{\text{PDE}}$\end{tabular}}}}%
    \put(0.47752843,0.00198726){\color[rgb]{0,0,0}\makebox(0,0)[lt]{\lineheight{1.25}\smash{\begin{tabular}[t]{l}$\theta^*$\end{tabular}}}}%
    \put(0.65058802,0.25668297){\color[rgb]{0,0,0}\makebox(0,0)[lt]{\lineheight{1.25}\smash{\begin{tabular}[t]{l}$\mathcal{L}_{\text{init}}$\end{tabular}}}}%
    \put(0.54455107,0.35427465){\color[rgb]{0,0,0}\makebox(0,0)[lt]{\lineheight{1.25}\smash{\begin{tabular}[t]{l}$\mathcal{L}_{\text{edge}}$\end{tabular}}}}%
    \put(0.3896584,0.52208443){\color[rgb]{0,0,0}\makebox(0,0)[lt]{\lineheight{1.25}\smash{\begin{tabular}[t]{l}$G_\theta^u(t, 0/1)$\end{tabular}}}}%
    \put(0.75714695,0.51982845){\color[rgb]{0,0,0}\makebox(0,0)[lt]{\lineheight{1.25}\smash{\begin{tabular}[t]{l}$G_\theta^u(0,x)$\end{tabular}}}}%
    \put(0,0){\includegraphics[width=\unitlength,page=1]{architecture.pdf}}%
    \put(0.39956108,1.71416889){\color[rgb]{0,0,0}\makebox(0,0)[lt]{\lineheight{1.25}\smash{\begin{tabular}[t]{l}$(u_{\mathrm{origin}},u_{\mathrm{target}}, u_{\mathrm{init}}, \nu)$\end{tabular}}}}%
    \put(0,0){\includegraphics[width=\unitlength,page=2]{architecture.pdf}}%
    \put(0.12737421,1.50297755){\color[rgb]{0,0,0}\makebox(0,0)[lt]{\lineheight{1.25}\smash{\begin{tabular}[t]{l}Trunk Net\end{tabular}}}}%
    \put(0.62347343,1.50297755){\color[rgb]{0,0,0}\makebox(0,0)[lt]{\lineheight{1.25}\smash{\begin{tabular}[t]{l}Branch Net\end{tabular}}}}%
    \put(0,0){\includegraphics[width=\unitlength,page=3]{architecture.pdf}}%
    \put(0.42562203,0.88385619){\color[rgb]{0,0,0}\makebox(0,0)[lt]{\lineheight{1.25}\smash{\begin{tabular}[t]{l}$G_\theta^u(t,x)$\end{tabular}}}}%
    \put(0,0){\includegraphics[width=\unitlength,page=4]{architecture.pdf}}%
    \put(0.79618183,1.46128828){\color[rgb]{0,0,0}\makebox(0,0)[lt]{\lineheight{1.25}\smash{\begin{tabular}[t]{l}{\tiny $b_1$}\end{tabular}}}}%
    \put(0,0){\includegraphics[width=\unitlength,page=5]{architecture.pdf}}%
    \put(0.27753144,1.46160379){\color[rgb]{0,0,0}\makebox(0,0)[lt]{\lineheight{1.25}\smash{\begin{tabular}[t]{l}{\tiny $s_1$}\end{tabular}}}}%
    \put(0,0){\includegraphics[width=\unitlength,page=6]{architecture.pdf}}%
    \put(0.10520448,0.44918232){\color[rgb]{0,0,0}\makebox(0,0)[lt]{\lineheight{1.25}\smash{\begin{tabular}[t]{l}PDE\end{tabular}}}}%
    \put(0.45654538,0.45135682){\color[rgb]{0,0,0}\makebox(0,0)[lt]{\lineheight{1.25}\smash{\begin{tabular}[t]{l}Flow\end{tabular}}}}%
    \put(0.51967144,0.11422548){\color[rgb]{0,0,0}\makebox(0,0)[lt]{\lineheight{1.25}\smash{\begin{tabular}[t]{l}min\end{tabular}}}}%
    \put(0.78245693,0.44989593){\color[rgb]{0,0,0}\makebox(0,0)[lt]{\lineheight{1.25}\smash{\begin{tabular}[t]{l}Initial\end{tabular}}}}%
    \put(0,0){\includegraphics[width=\unitlength,page=7]{architecture.pdf}}%
    \put(0.44295044,0.21212603){\color[rgb]{0,0,0}\makebox(0,0)[lt]{\lineheight{1.25}\smash{\begin{tabular}[t]{l}Loss\end{tabular}}}}%
    \put(0,0){\includegraphics[width=\unitlength,page=8]{architecture.pdf}}%
  \end{picture}%
\endgroup%

%% file: icml2025_main_revision.bbl
\begin{thebibliography}{59}
\providecommand{\natexlab}[1]{#1}
\providecommand{\url}[1]{\texttt{#1}}
\expandafter\ifx\csname urlstyle\endcsname\relax
  \providecommand{\doi}[1]{doi: #1}\else
  \providecommand{\doi}{doi: \begingroup \urlstyle{rm}\Url}\fi

\bibitem[Arioli \& Benzi(2018)Arioli and Benzi]{arioli2018finite}
Arioli, M. and Benzi, M.
\newblock A finite element method for quantum graphs.
\newblock \emph{IMA Journal of Numerical Analysis}, 38\penalty0 (3):\penalty0
  1119--1163, 2018.

\bibitem[Barab{\'a}si(2016)]{barabasi2013network}
Barab{\'a}si, A.-L.
\newblock \emph{Network science}.
\newblock Number 1987. Cambridge University Press, Cambridge, 2016.

\bibitem[Berkolaiko \& Kuchment(2013)Berkolaiko and
  Kuchment]{berkolaiko2013introduction}
Berkolaiko, G. and Kuchment, P.
\newblock \emph{Introduction to quantum graphs}.
\newblock Number 186. American Mathematical Soc., Providence, R.I., 2013.

\bibitem[Blechschmidt \& Ernst(2021)Blechschmidt and
  Ernst]{blechschmidt2021three}
Blechschmidt, J. and Ernst, O.~G.
\newblock Three ways to solve partial differential equations with neural
  networks—a review.
\newblock \emph{GAMM-Mitteilungen}, 44\penalty0 (2):\penalty0 e202100006, 2021.

\bibitem[Blechschmidt et~al.(2022)Blechschmidt, Pietschman, Riemer, Stoll, and
  Winkler]{Blechschmidt2022}
Blechschmidt, J., Pietschman, J.-F., Riemer, T.-C., Stoll, M., and Winkler, M.
\newblock A comparison of pinn approaches for drift-diffusion equations on
  metric graphs, 2022.

\bibitem[Bradbury et~al.(2018)Bradbury, Frostig, Hawkins, Johnson, Leary,
  Maclaurin, Necula, Paszke, Vander{P}las, Wanderman-{M}ilne, and
  Zhang]{jax2018github}
Bradbury, J., Frostig, R., Hawkins, P., Johnson, M.~J., Leary, C., Maclaurin,
  D., Necula, G., Paszke, A., Vander{P}las, J., Wanderman-{M}ilne, S., and
  Zhang, Q.
\newblock {JAX}: composable transformations of {P}ython+{N}um{P}y programs,
  2018.
\newblock URL \url{http://github.com/jax-ml/jax}.

\bibitem[Bressloff \& Karamched(2016)Bressloff and Karamched]{Bressloff2016}
Bressloff, P.~C. and Karamched, B.~R.
\newblock Model of reversible vesicular transport with exclusion.
\newblock \emph{Journal of Physics A: Mathematical and Theoretical},
  49\penalty0 (34):\penalty0 345602, jul 2016.
\newblock \doi{10.1088/1751-8113/49/34/345602}.
\newblock URL \url{https://dx.doi.org/10.1088/1751-8113/49/34/345602}.

\bibitem[Bressloff \& Levien(2015)Bressloff and Levien]{Bressloff2015}
Bressloff, P.~C. and Levien, E.
\newblock Synaptic democracy and vesicular transport in axons.
\newblock \emph{Physical Review Letters}, 114\penalty0 (16), 2015.
\newblock ISSN 1079-7114.
\newblock \doi{10.1103/physrevlett.114.168101}.
\newblock URL \url{http://dx.doi.org/10.1103/PhysRevLett.114.168101}.

\bibitem[Brezis(2010)]{Brezis2010}
Brezis, H.
\newblock \emph{Functional Analysis, Sobolev Spaces and Partial Differential
  Equations}.
\newblock Springer Science and Business Media, Berlin Heidelberg, 2010.
\newblock ISBN 978-0-387-70913-0.

\bibitem[Burger \& Pietschmann(2016)Burger and Pietschmann]{Burger2016}
Burger, M. and Pietschmann, J.-F.
\newblock {Flow} characteristics in a crowded transport model.
\newblock \emph{Nonlinearity}, 29:\penalty0 3528--3550, 2016.
\newblock \doi{10.1088/0951-7715/29/11/3528}.
\newblock WWU::123155.

\bibitem[Burger et~al.(2020)Burger, Humpert, and Pietschmann]{Burger2020}
Burger, M., Humpert, I., and Pietschmann, J.-F.
\newblock On {F}okker-{P}lanck equations with in- and outflow of mass.
\newblock \emph{Kinetic \& Related Models}, 13:\penalty0 249--277, 01 2020.
\newblock \doi{10.3934/krm.2020009}.

\bibitem[Chen et~al.(2021)Chen, Duan, and Karniadakis]{chen2019learning}
Chen, X., Duan, J., and Karniadakis, G.~E.
\newblock Learning and meta-learning of stochastic advection-diffusion-reaction
  systems from sparse measurements.
\newblock \emph{European Journal of Applied Mathematics}, 32\penalty0
  (3):\penalty0 397--420, 2021.
\newblock ISSN 0956-7925.

\bibitem[Coclite et~al.(2005)Coclite, Garavello, and Piccoli]{Coclite2005}
Coclite, G.~M., Garavello, M., and Piccoli, B.
\newblock Traffic flow on a road network.
\newblock \emph{SIAM Journal on Mathematical Analysis}, 36\penalty0
  (6):\penalty0 1862--1886, 2005.
\newblock \doi{10.1137/S0036141004402683}.

\bibitem[Crossley et~al.(2025)Crossley, Pietschmann, and
  Schmidtchen]{Crossley2025}
Crossley, R.~M., Pietschmann, J.-F., and Schmidtchen, M.
\newblock Existence of weak solutions for a volume-filling model of cell
  invasion into extracellular matrix.
\newblock \emph{Journal of Differential Equations}, 428:\penalty0 721--746,
  2025.
\newblock ISSN 0022-0396.
\newblock \doi{https://doi.org/10.1016/j.jde.2025.02.023}.
\newblock URL
  \url{https://www.sciencedirect.com/science/article/pii/S0022039625001421}.

\bibitem[Cuomo et~al.(2022)Cuomo, Di~Cola, Giampaolo, Rozza, Raissi, and
  Piccialli]{cuomo2022scientific}
Cuomo, S., Di~Cola, V.~S., Giampaolo, F., Rozza, G., Raissi, M., and Piccialli,
  F.
\newblock Scientific machine learning through physics--informed neural
  networks: Where we are and what’s next.
\newblock \emph{Journal of Scientific Computing}, 92\penalty0 (3):\penalty0 88,
  2022.

\bibitem[De~los Reyes(2015)]{de2015numerical}
De~los Reyes, J.~C.
\newblock \emph{Numerical PDE-constrained optimization}.
\newblock Springer, 2015.

\bibitem[Domschke et~al.(2021)Domschke, Hiller, Lang, Mehrmann, Morandin, and
  Tischendorf]{DomschkeHillerLangetal2021}
Domschke, P., Hiller, B., Lang, J., Mehrmann, V., Morandin, R., and
  Tischendorf, C.
\newblock Gas network modeling: An overview.
\newblock 2021.
\newblock URL \url{https://opus4.kobv.de/opus4-trr154/411}.

\bibitem[Gomes et~al.(2019)Gomes, Stuart, and Wolfram]{Gomes2019}
Gomes, S.~N., Stuart, A.~M., and Wolfram, M.-T.
\newblock Parameter estimation for macroscopic pedestrian dynamics models from
  microscopic data.
\newblock \emph{SIAM Journal on Applied Mathematics}, 79\penalty0 (4):\penalty0
  1475--1500, 2019.
\newblock \doi{10.1137/18M1215980}.

\bibitem[Gyrya \& Zlotnik(2019)Gyrya and Zlotnik]{gyrya2019explicit}
Gyrya, V. and Zlotnik, A.
\newblock An explicit staggered-grid method for numerical simulation of
  large-scale natural gas pipeline networks.
\newblock \emph{Applied Mathematical Modelling}, 65:\penalty0 34--51, 2019.

\bibitem[Haghighat et~al.(2020)Haghighat, Raissi, Moure, Gomez, and
  Juanes]{haghighat2020deep}
Haghighat, E., Raissi, M., Moure, A., Gomez, H., and Juanes, R.
\newblock A deep learning framework for solution and discovery in solid
  mechanics.
\newblock \emph{arXiv:2003.02751}, 2020.

\bibitem[Heinlein et~al.(2021)Heinlein, Klawonn, Lanser, and
  Weber]{Heinlein2021}
Heinlein, A., Klawonn, A., Lanser, M., and Weber, J.
\newblock Combining machine learning and domain decomposition methods for the
  solution of partial differential equations -- a review.
\newblock \emph{GAMM-Mitteilungen}, 44\penalty0 (1), 2021.

\bibitem[Hinze et~al.(2011)Hinze, Kunkel, and Vierling]{hinze2011pod}
Hinze, M., Kunkel, M., and Vierling, M.
\newblock Pod model order reduction of drift-diffusion equations in electrical
  networks.
\newblock In \emph{Model Reduction for Circuit Simulation}, volume~74, pp.\
  177--192. Springer, Dordrecht, 2011.

\bibitem[Jagtap \& Karniadakis(2020)Jagtap and
  Karniadakis]{JagtapKardiadakis2020}
Jagtap, A. and Karniadakis, G.
\newblock Extended physics-informed neural networks ({XPINN}s): {A} generalized
  space-time domain decomposition based deep learning framework for nonlinear
  partial differential equations.
\newblock \emph{Communications in Computational Physics}, 28:\penalty0
  2002--2041, 11 2020.

\bibitem[Jagtap et~al.(2020)Jagtap, Kharazmi, and
  Karniadakis]{jagtap2020conservative}
Jagtap, A.~D., Kharazmi, E., and Karniadakis, G.~E.
\newblock Conservative physics-informed neural networks on discrete domains for
  conservation laws: Applications to forward and inverse problems.
\newblock \emph{Computer Methods in Applied Mechanics and Engineering}, 365,
  2020.

\bibitem[Jin et~al.(2021)Jin, Cai, Li, and Karniadakis]{jin2021nsfnets}
Jin, X., Cai, S., Li, H., and Karniadakis, G.~E.
\newblock Nsfnets (navier-stokes flow nets): Physics-informed neural networks
  for the incompressible navier-stokes equations.
\newblock \emph{Journal of Computational Physics}, 426, 2021.

\bibitem[Kingma \& Ba(2015)Kingma and Ba]{kingma2014adam}
Kingma, D.~P. and Ba, J.
\newblock Adam: A method for stochastic optimization.
\newblock In Bengio, Y. and LeCun, Y. (eds.), \emph{ICLR (Poster)}, 2015.

\bibitem[Kovachki et~al.(2023)Kovachki, Li, Liu, Azizzadenesheli, Bhattacharya,
  Stuart, and Anandkumar]{kovachki2023neural}
Kovachki, N., Li, Z., Liu, B., Azizzadenesheli, K., Bhattacharya, K., Stuart,
  A., and Anandkumar, A.
\newblock Neural operator: Learning maps between function spaces with
  applications to pdes.
\newblock \emph{Journal of Machine Learning Research}, 24\penalty0
  (89):\penalty0 1--97, 2023.

\bibitem[Lagnese et~al.(2012)Lagnese, Leugering, and
  Schmidt]{lagnese2012modeling}
Lagnese, J.~E., Leugering, G., and Schmidt, E.~G.
\newblock \emph{Modeling, analysis and control of dynamic elastic multi-link
  structures}.
\newblock Birkhäuser, Boston, 2012.

\bibitem[Lazarov et~al.(1996)Lazarov, Mishev, and
  Vassilevski]{LazarovMishevVassilevski1996}
Lazarov, R.~D., Mishev, I.~D., and Vassilevski, P.~S.
\newblock Finite volume methods for convection-diffusion problems.
\newblock \emph{{SIAM} Journal on Numerical Analysis}, 33\penalty0
  (1):\penalty0 31--55, feb 1996.

\bibitem[Leugering(2017)]{leugering2017domain}
Leugering, G.
\newblock Domain decomposition of an optimal control problem for semi-linear
  elliptic equations on metric graphs with application to gas networks.
\newblock \emph{Applied Mathematics}, 8\penalty0 (08):\penalty0 1074--1099,
  2017.

\bibitem[LeVeque(2002)]{leveque2002finite}
LeVeque, R.~J.
\newblock \emph{Finite volume methods for hyperbolic problems}, volume~31.
\newblock Cambridge university press, 2002.

\bibitem[Li et~al.(2020)Li, Kovachki, Azizzadenesheli, Liu, Bhattacharya,
  Stuart, and Anandkumar]{li2020neural}
Li, Z., Kovachki, N., Azizzadenesheli, K., Liu, B., Bhattacharya, K., Stuart,
  A., and Anandkumar, A.
\newblock Neural operator: Graph kernel network for partial differential
  equations.
\newblock \emph{arXiv preprint arXiv:2003.03485}, 2020.

\bibitem[Loder et~al.(2019)Loder, Amb\"{u}hl, Menendez, and
  Axhausen]{Loder2019}
Loder, A., Amb\"{u}hl, L., Menendez, M., and Axhausen, K.~W.
\newblock Understanding traffic capacity of urban networks.
\newblock \emph{Scientific Reports}, 9\penalty0 (1), November 2019.
\newblock ISSN 2045-2322.
\newblock \doi{10.1038/s41598-019-51539-5}.

\bibitem[Lu et~al.(2021)Lu, Jin, Pang, Zhang, and Karniadakis]{lu2021learning}
Lu, L., Jin, P., Pang, G., Zhang, Z., and Karniadakis, G.~E.
\newblock Learning nonlinear operators via deeponet based on the universal
  approximation theorem of operators.
\newblock \emph{Nature machine intelligence}, 3\penalty0 (3):\penalty0
  218--229, 2021.

\bibitem[Lye et~al.(2020)Lye, Mishra, and Ray]{lye2020deep}
Lye, K.~O., Mishra, S., and Ray, D.
\newblock Deep learning observables in computational fluid dynamics.
\newblock \emph{Journal of Computational Physics}, 410, 2020.

\bibitem[Magiera et~al.(2020)Magiera, Ray, Hesthaven, and
  Rohde]{magiera2020constraint}
Magiera, J., Ray, D., Hesthaven, J.~S., and Rohde, C.
\newblock Constraint-aware neural networks for {R}iemann problems.
\newblock \emph{Journal of Computational Physics}, 409, 2020.

\bibitem[Mao et~al.(2020)Mao, Jagtap, and Karniadakis]{mao2020physics}
Mao, Z., Jagtap, A.~D., and Karniadakis, G.~E.
\newblock Physics-informed neural networks for high-speed flows.
\newblock \emph{Computer Methods in Applied Mechanics and Engineering}, 360,
  2020.
\newblock \doi{https://doi.org/10.1016/j.cma.2019.112789}.

\bibitem[Meng \& Karniadakis(2020)Meng and Karniadakis]{meng2020composite}
Meng, X. and Karniadakis, G.~E.
\newblock A composite neural network that learns from multi-fidelity data:
  {A}pplication to function approximation and inverse {PDE} problems.
\newblock \emph{Journal of Computational Physics}, 401, 2020.

\bibitem[Misyris et~al.(2020)Misyris, Venzke, and
  Chatzivasileiadis]{misyris2020physics}
Misyris, G.~S., Venzke, A., and Chatzivasileiadis, S.
\newblock Physics-informed neural networks for power systems.
\newblock In \emph{2020 IEEE Power \& Energy Society General Meeting (PESGM)},
  pp.\  1--5. IEEE, 2020.

\bibitem[Morton et~al.(1997)Morton, Stynes, and Süli]{MortonStynesSuli1997}
Morton, K., Stynes, M., and Süli, E.
\newblock Analysis of a cell-vertex finite volume method for
  convection-diffusion problems.
\newblock \emph{Mathematics of Computation}, 66\penalty0 (220):\penalty0
  1389--1406, 1997.

\bibitem[Newman(2018)]{newman2018networks}
Newman, M.
\newblock \emph{Networks}.
\newblock Oxford university press, Oxford, 2018.

\bibitem[Nguyen-Thanh et~al.(2020)Nguyen-Thanh, Zhuang, and
  Rabczuk]{nguyen2020deep}
Nguyen-Thanh, V.~M., Zhuang, X., and Rabczuk, T.
\newblock A deep energy method for finite deformation hyperelasticity.
\newblock \emph{European Journal of Mechanics-A/Solids}, 80, 2020.

\bibitem[Pang et~al.(2019)Pang, Lu, and Karniadakis]{pang2019fpinns}
Pang, G., Lu, L., and Karniadakis, G.~E.
\newblock {fPINNs}: {F}ractional physics-informed neural networks.
\newblock \emph{SIAM Journal on Scientific Computing}, 41\penalty0
  (4):\penalty0 A2603--A2626, 2019.

\bibitem[Piccoli \& Garavello(2006)Piccoli and Garavello]{piccoli2006traffic}
Piccoli, B. and Garavello, M.
\newblock Traffic flow on networks.
\newblock \emph{American Institute of Mathematical Sciences}, 2006.

\bibitem[Rackauckas et~al.(2020)Rackauckas, Ma, Martensen, Warner, Zubov,
  Supekar, Skinner, Ramadhan, and Edelman]{rackauckas2020universal}
Rackauckas, C., Ma, Y., Martensen, J., Warner, C., Zubov, K., Supekar, R.,
  Skinner, D., Ramadhan, A., and Edelman, A.
\newblock Universal differential equations for scientific machine learning.
\newblock \emph{arXiv preprint arXiv:2001.04385}, 2020.

\bibitem[Raissi et~al.(2018)Raissi, Yazdani, and Karniadakis]{raissi2018hidden}
Raissi, M., Yazdani, A., and Karniadakis, G.~E.
\newblock Hidden fluid mechanics: {A} navier-stokes informed deep learning
  framework for assimilating flow visualization data.
\newblock \emph{arXiv:1808.04327}, 2018.

\bibitem[Raissi et~al.(2019)Raissi, Perdikaris, and
  Karniadakis]{raissi2019physics}
Raissi, M., Perdikaris, P., and Karniadakis, G.~E.
\newblock Physics-informed neural networks: A deep learning framework for
  solving forward and inverse problems involving nonlinear partial differential
  equations.
\newblock \emph{Journal of Computational physics}, 378:\penalty0 686--707,
  2019.

\bibitem[Rao et~al.(2020)Rao, Sun, and Liu]{rao2020physics}
Rao, C., Sun, H., and Liu, Y.
\newblock Physics informed deep learning for computational elastodynamics
  without labeled data.
\newblock \emph{arXiv:2006.08472}, 2020.

\bibitem[Sahli~Costabal et~al.(2020)Sahli~Costabal, Yang, Perdikaris, Hurtado,
  and Kuhl]{sahli2020physics}
Sahli~Costabal, F., Yang, Y., Perdikaris, P., Hurtado, D.~E., and Kuhl, E.
\newblock Physics-informed neural networks for cardiac activation mapping.
\newblock \emph{Frontiers in Physics}, 8:\penalty0 42, 2020.

\bibitem[Seo et~al.(2017)Seo, Bayen, Kusakabe, and
  Asakura]{Seo2017_trafficflowestimation}
Seo, T., Bayen, A.~M., Kusakabe, T., and Asakura, Y.
\newblock Traffic state estimation on highway: A comprehensive survey.
\newblock \emph{Annual Reviews in Control}, 43:\penalty0 128--151, 2017.
\newblock ISSN 1367-5788.
\newblock \doi{https://doi.org/10.1016/j.arcontrol.2017.03.005}.

\bibitem[Simon(1986)]{simon1986compact}
Simon, J.
\newblock Compact sets in the space {$L^p (O, T; B)$}.
\newblock \emph{Annali di Matematica pura ed applicata}, 146:\penalty0 65--96,
  1986.

\bibitem[Stoll \& Winkler(2021)Stoll and Winkler]{stoll2021optimal}
Stoll, M. and Winkler, M.
\newblock Optimal dirichlet control of partial differential equations on
  networks.
\newblock \emph{Electronic Transactions on Numerical Analysis}, 54:\penalty0
  392--419, 2021.

\bibitem[ten Thije~Boonkkamp \& Anthonissen(2010)ten Thije~Boonkkamp and
  Anthonissen]{ThijeBoonkkampAnthonissen2010}
ten Thije~Boonkkamp, J. H.~M. and Anthonissen, M. J.~H.
\newblock The finite volume-complete flux scheme
  for~advection-diffusion-reaction equations.
\newblock \emph{Journal of Scientific Computing}, 46\penalty0 (1):\penalty0
  47--70, jun 2010.

\bibitem[Thiyagalingam et~al.(2022)Thiyagalingam, Shankar, Fox, and
  Hey]{thiyagalingam2022scientific}
Thiyagalingam, J., Shankar, M., Fox, G., and Hey, T.
\newblock Scientific machine learning benchmarks.
\newblock \emph{Nature Reviews Physics}, 4\penalty0 (6):\penalty0 413--420,
  2022.

\bibitem[Wang et~al.(2021)Wang, Wang, and Perdikaris]{wang2021learning}
Wang, S., Wang, H., and Perdikaris, P.
\newblock Learning the solution operator of parametric partial differential
  equations with physics-informed {DeepONets}.
\newblock \emph{Science advances}, 7\penalty0 (40):\penalty0 eabi8605, 2021.

\bibitem[Wessels et~al.(2020)Wessels, Wei{\ss}enfels, and
  Wriggers]{wessels2020neural}
Wessels, H., Wei{\ss}enfels, C., and Wriggers, P.
\newblock The neural particle method--an updated lagrangian physics informed
  neural network for computational fluid dynamics.
\newblock \emph{Computer Methods in Applied Mechanics and Engineering}, 368,
  2020.

\bibitem[Yang et~al.(2020)Yang, Zhang, and Karniadakis]{yang2020physics}
Yang, L., Zhang, D., and Karniadakis, G.~E.
\newblock Physics-informed generative adversarial networks for stochastic
  differential equations.
\newblock \emph{SIAM Journal on Scientific Computing}, 42\penalty0
  (1):\penalty0 A292--A317, 2020.

\bibitem[Yin et~al.(2022)Yin, Zhang, Yu, and Karniadakis]{yin2022interfacing}
Yin, M., Zhang, E., Yu, Y., and Karniadakis, G.~E.
\newblock Interfacing finite elements with deep neural operators for fast
  multiscale modeling of mechanics problems.
\newblock \emph{Computer methods in applied mechanics and engineering},
  402:\penalty0 115027, 2022.

\bibitem[Zhu et~al.(2019)Zhu, Zabaras, Koutsourelakis, and
  Perdikaris]{zhu2019physics}
Zhu, Y., Zabaras, N., Koutsourelakis, P.-S., and Perdikaris, P.
\newblock Physics-constrained deep learning for high-dimensional surrogate
  modeling and uncertainty quantification without labeled data.
\newblock \emph{Journal of Computational Physics}, 394:\penalty0 56--81, 2019.

\end{thebibliography}
